\documentclass[journal]{IEEEtran}
\usepackage[T1]{fontenc}
\usepackage{graphicx,colordvi,psfrag}
\usepackage{amsmath,amssymb,algpseudocode}
\usepackage{algorithm}
\usepackage{algorithmicx}
\usepackage{calc,pstricks, pgf, xcolor}
\usepackage{bbding}
\usepackage{bbm}
\usepackage{dsfont}
\usepackage{stfloats}
\usepackage{hyperref}
\usepackage{xparse}
\usepackage{enumerate}
\usepackage{tikz}
\usetikzlibrary{shapes,arrows,positioning,calc}

\usepackage{subcaption}

\newcommand{\ZZ}{\mathbb{Z}}
\newcommand{\RR}{\mathbb{R}}

\newcommand{\diag}{\mathop{\mathrm{diag}}}

\newcommand{\trace}{\mathop{\mathrm{tr}}}
\def\tr{\trace}

\newcommand{\Vol}{\mathrm{Vol}}

\newcommand{\eps}{\varepsilon}

\newcommand{\Unif}{\mathrm{Uniform}}
\newcommand\indep{\protect\mathpalette{\protect\independenT}{\perp}}
\def\independenT#1#2{\mathrel{\rlap{$#1#2$}\mkern2mu{#1#2}}}

\def\EE{\mathbb{E}}
\newcommand{\m}{\mathcal}

\newcommand{\HR}{\mathrm{High-Rate}}
\newcommand{\covol}{\mathrm{covol}}

\NewDocumentCommand{\DIP}{e{^_}}{D^{\mathrm{IP}\IfValueT{#1}{,#1}}_{\IfValueT{#2}{#2}}}

\NewDocumentCommand{\DMM}{e{^_}}{D^{\mathrm{MM}\IfValueT{#1}{,#1}}_{\IfValueT{#2}{#2}}}

\DeclareMathOperator*{\argmin}{\arg\!\min}

\newcommand{\nbyp}[1]{{\color{red}[ \textbf{YP:} #1}]}

\newtheorem{proposition}{Proposition}

\newtheorem{lemma}{Lemma}

\newenvironment{proof}[1][Proof]{\noindent\textbf{#1.} }{\ \rule{0.5em}{0.5em}}

\begin{document}
\title{High-Rate Quantized Matrix Multiplication II}

\author{Or Ordentlich 
and	Yury Polyanskiy 
\thanks{O. Ordentlich is with the 
Hebrew University of Jerusalem, Israel (\texttt{or.ordentlich@mail.huji.ac.il}).  Y. Polyanskiy is with the MIT,
USA (\texttt{yp@mit.edu}). The work of OO was supported by the Israel Science 
Foundation (ISF),
grant No. 2878/25. The work of YP was supported in part by the MIT-IBM Watson AI Lab and by the National Science
    Foundation under Grant No CCF-2131115. }
}

\date{}

\maketitle
\begin{abstract}
This is the second part of the work investigating quantized matrix multiplication (MatMul). 
In part I we considered the case of calibration-free quantization, whereas here we discuss the
setting where covariance matrix $\Sigma_X$ of the columns of the second factor is available. This
setting arises in the ubiquitous task of \textit{weight-only} post-training quantization of LLMs.
 
Weight-only quantization is related to the problem of weighted mean squared error (WMSE) source
coding, whose classical (reverse) waterfilling solution dictates how one should distribute rate
between coordinates of the vector. We show how waterfilling can be used to improve practical LLM
quantization algorithms (GPTQ), which at present allocate rate equally. A recent scheme (known as ``WaterSIC'') 
that only uses scalar INT quantizers is analyzed and its high-rate
performance is shown to be (a) basis free (i.e., characterized by the determinant of $\Sigma_X$ and, thus, unlike existing
schemes, is immune to applying random rotations);  and (b)  within a multiplicative
factor of $\frac{2\pi e}{12}$ (or 0.25 bit/entry) of the information-theoretic 
distortion limit. GPTQ's performance, in turn, is affected by the choice of basis, but for a random
rotation and actual $\Sigma_X$ from Llama-3-8B we find
it to be within 0.1 bit (depending on the layer type) of WaterSIC, suggesting that GPTQ with
random rotation is also near optimal, at least in the high-rate regime. 
\end{abstract}


\section{Introduction}

Matrix multiplication (MatMul) constitutes the main building block of
AI, and quantization for matrix multiplication has become a central technique for increasing its
efficiency. In Part I~\cite{OP26part1} we discussed the problem of quantizing two generic matrices
with the goal of approximating their product with minimal distortion. The fundamental limits for
this assumptions-free setup were derived in~\cite{ordentlich2024optimal}, which we compared with
the performance of popular generic quantization schemes for MatMul (such as ubiquitous absmax
round-to-nearest with INT or FP arithmetic).

Generic quantized MatMul schemes require no prior assumptions on the matrices to be multiplied,
and their performance is consequently robust. However, in LLMs activations often have a rather
special low-dimensional structure, obviating the need for orthogonalizing quantization errors with
the principal directions of activation variation~\cite{adaround,frantar2023optq}. The process of
collecting statistics of activation is known as \emph{calibration}, and thus we can think of Part
I as focusing on calibration-free and Part on calibration-aided methods. Clearly, the goal of such
methods is to improve matrix product approximation by leveraging calibration statistics.

Furthermore, in this paper we focus on the case where arithmetic (actual matmul) is executed and
activations loaded in
full precision. The goal, thus, is to load weight matrices in compressed (low-resolution) format
from the high bandwidth memory (HBM) and then dequantize them into full-precision matrix before
executing the actual matmul. 
Compared to full-precision (BF16) weights, quantizing those to $R$ bits per parameter reduces the required HBM storage and
bandwidth capacity by factor
$16/R$. 

In Section~\ref{sec:weightonly_theory} we study the fundamental limits of \emph{weight-only quantization} where we are interested in the
matrix product $X^\top W$ and only $W$ needs to be quantized to $\hat{W}$, whereas $X$ is given
in full resolution. The problem is made interesting through the fact that while quantization of
$W$ cannot know $X$, it has access to covariance matrix $\Sigma_X$ of (the columns of) $X$.
We show that, in essence, using standard ``isotropic'' codebook and $\Sigma_X$-oblivious
quantization, one attains distortion $D\approx
({1\over n} \tr \Sigma_X) 2^{-2R}$. At the same time, the classical waterfilling
formula says that optimal codebook (with optimal quantization) can attain $D\approx
 |\Sigma_X|^{1/n} 2^{-2R}$, with the difference given by the gap in AM-GM inequality for
eigenvalues of $\Sigma_X$. It turns out that even with isotropic codebook, but $\Sigma_X$-aware
encoding one can still attain the same performance.

In Section~\ref{sec:weightonly_practice} we discuss the practical aspects of the weights-only quantization problem. Low-complexity algorithms for weight-only quantization in LLMs have been pioneered
by~\cite{dettmers2022gpt3} (a $\Sigma_X$-oblivious rounding to nearest integer),~\cite{adaround}, and
GPTQ~\cite{frantar2023optq}. The latter was originally derived from~\cite{hassibi1993optimal}, but
was soon found to be equivalent to classical algorithms known as \emph{successive
interference cancellation (SIC)} in communication and Babai's algorithm in computer science, cf.
Section~\ref{sec:weightonly_practice}. We show that GPTQ (that uses constant rate per entry)
attains distortion decaying as $D\approx {2\pi e\over
12} ({1\over n} \sum_{i=1}^n U_{i,i}^2) 2^{-2R}$, where $U^\top U = \Sigma_X$ is the upper-triangular
(Cholesky) decomposition of $\Sigma_X$. At the same time, by adjusting rate per-coordinate
(``WaterSIC'' algorithm~\cite{lifar2026watersic}) one can
attain $D\approx {2\pi e\over
12} (\prod_{i=1}^n U_{i,i}^2)^{1\over n} 2^{-2R}$. Since $\prod_{i=1}^n U_{i,i} = \sqrt{|\Sigma_X|}$ this
performance is only a factor away from information-theoretically optimal, or within a rate gap of
at most 0.25 bit. 

Figures~\ref{fig:llama_ev} and~\ref{fig:llama_chol} exemplify rate gains for all these four (two
theoretical and two practical) algorithms on the example of layers of Llama-3-8B and wiki2
calibration data.

\section{Weight quantization: Theory}
\label{sec:weightonly_theory}

Let us restrict attention to a particular linear layer in the network with weight matrix $W\in
\RR^{n\times a}$. The linear layer operates by taking (a vector of) input activations $X$ and
producing
$$ Y = X^\top W\,.$$
For that, $W$ needs to be loaded from memory and this results in the crucial
bottleneck during generation/reasoning part of the LLM operation. Thus, our job is to replace $W$
with quantized version $\hat W$, which can be loaded much faster, with the target of minimizing
\textit{expected} distortion over random inputs $X$, that is we aim to minimize
\begin{align}\label{eq:wot_1}
	D&={1\over n} \EE\|X^\top(W-\hat{W})\|_F^2 \nonumber\\
    &=\frac{1}{n} \tr  (W-\hat W)^\top \Sigma_X (W-\hat W)\,,
\end{align}
where $\Sigma_X=\EE[X X^\top]$ (This objective is by far the most popular one, though, others
exist~\cite{li2021brecq,tseng2025model,badri2023hqq}.)

Since only the second order statistics $\Sigma_X$ affect this objective, in practice one uses
a set of samples (known as \emph{calibration data}) from the underlying distribution in order to
estimate it. With the second order statistics at hand, the problem of quantizing $W$ becomes a
weighted mean squared error (WMSE) quantization problem. We note that calibration data about $X$ is
also useful for quantization of both weights and activations (i.e. the setting of~\cite{OP26part1}). See e.g.~\cite{xiao2023smoothquant},~\cite[Section
4.5]{NestQuant},~\cite{zhang2025qronos},~\cite{zhang2025provable}. 

Weight matrices start life as iid Gaussian and evolve during training. However, in many ways their
statistics are still largely very similar to iid Gaussian~\cite[Appendix]{lifar2026watersic} and that is how we will model them in this
(theoretical) section. Some (mostly older) LLMs have peculiar aspects of weight matrices, such as
existence of outliers and rank deficiencies, that need to be exploited/mitigated in practical
algorithms, though.

Notice also that objective~\eqref{eq:wot_1} treats distortion across columns of $W$ in an additive
manner. Consequently, analyzing the case of $a=1$ (single column, or single output neuron) can be
done without loss of generality, at least as long as $n\gg1$ and full advantage of vector
quantization can be exploited already in dimension $n$, without needing to do joint vector
quantization across all $na$ dimensions.

So in summary, the goal of this section is to understand theoretically the problem of mapping
$W \sim \mathcal{N}(0, I_n)$ to $\hat W$ belonging to the universe of $2^{nR}$ possibilities, with
the goal of minimizing
\begin{align}
D = {1\over n} (W-\hat W)^\top \Sigma_X (W-\hat W)\,.
\label{eq:DwmseDef}
\end{align}

One important aspect of the weight only quantization problem worth pointing out from the outset is
the asymmetry between encoding and decoding. The encoding (quantization) is done offline, once for
each model, and may therefore be computationally demanding. The decoding (de-quantization), on the
other hand, is executed online every time a weight matrix is used, and must therefore be highly
efficient. (This is in contrast with activation quantization, which is done online and has to be
extremely computationally efficient.) In particular, while information about $\Sigma_X$ is
available to the encoder (recall that it operates offline and can be quite slow), the decoder's
online operation requires all information about $\Sigma_X$ be packaged inside the rate-constrained
description of $W$. Thus, assuming decoder knows $\Sigma_X$ or its SVD basis (rather than only its
spectrum) is not reflecting the practical constraints adequately.

Nevertheless, below we will start by treating the impractical \textit{fully informed case},
which will provide a firm lower bound to the more practical version of \textit{uninformed/oblivious decoder
case}. We initiate our discussion with a simple heuristic demonstration of the main points on the
example of scalar quantization.

\subsection{Scalar quantization and waterfilling}

Before discussing information theoretic results, let us consider a very simple case of $X$ with
uncorrelated coordinates, i.e. $\Sigma_X = \diag\{\lambda_1,\ldots,\lambda_n\}$. Suppose, in
addition, that the weights satisfy $W_i\in\Unif[-1/2,1/2)$.

Consider, first, the case of usual uniform quantization using $\epsilon$-grid, where
$\epsilon=2^{-R}$. The resulting distortion (in accordance with additive noise approximation discussed in Part I~\cite{OP26part1}) will be
$$ D_{1}(R) = {1\over n}\sum_{i=1}^n \lambda_i {\epsilon^2\over 12} = {\bar \lambda_{A}\over 12}
2^{-2R}\,,$$
where $\bar \lambda_A = {1\over n} \sum_{i} \lambda_i$ is the arithmetic mean of $\lambda$'s.

Now, given that different coordinates of $W$ contribute differently to $D$, one may naturally try
to allocate rate more judiciously. Specifically, let us solve the problem of
$$ D_2 = \min {1\over 12n} \sum_{i=1}^n \lambda_i \epsilon_i^2\,,$$
subject to rate constraint $\prod {1\over \epsilon_i} \le 2^{nR}$. Using Lagrange multipliers we
can easily find a parametric formula for the optimizer, given in terms of parameter $\tau > 0$:
\begin{align*} D_2 &= {1\over n} \sum_{i=1}^n \min(\tau, \lambda_i/12)\\
   R &= {1\over 2n} \sum_{i=1}^n \log \max(1, {\lambda_i\over 12\tau})\,,
\end{align*}
where the grid spacings $\epsilon_i = \min(\sqrt{12\tau\over \lambda_i}, 1)$.
This kind of allocation is traditionally called (reverse) \emph{waterfilling solution}, due to
interpretation of $\tau$ as water level that reduces $i$'th coordinate baseline
distortion (equal to $\lambda_i/12$) down to $\tau$. Those coordinates that are so insignificant
(``below waterlevel'') get zero rate allocation and are quantized to zero. 

In accordance with our focus on high-rate case, we can assume that $\tau < \min_i \lambda_i$ and get
a simplified expression:
$$ D_2(R) =  {\bar \lambda_G\over 12} 2^{-2R}\,,$$
where $\bar \lambda_G = (\prod_i \lambda_i)^{1/n}$ is the geometric mean of $\lambda_i$'s. 

In all, we conclude that when quantizing vectors under weighted quadratic metric, one should
choose quantization grid spacing $\epsilon_i \propto {1\over \sqrt{\lambda_i}}$ with coefficient
of proportionality chosen depending on the target required rate $R$. How much does one win from
this optimization is given by the AM-GM inequality $\bar \lambda_G < \bar \lambda_A$. Another way
to put this is that optimizing per-coordinate quantizers one wins about ${1\over 2} \log {\bar \lambda_A\over \bar \lambda_G}$ bits of rate. 

We will return to this idea below, when we describe a practical WaterSIC algorithm for weight-only
quantization in Section~\ref{sec:watersic}. There the role of $\lambda_i$'s is played by the
diagonal elements of the Cholesky decomposition of (non-diagonal) $\Sigma_X$. 

Fig.~\ref{fig:llama_ev} illustrates rate-advantage for input activations to various linear layers
of Llama-3-8B.

\begin{figure*}[t] 
    \centering
    
    \begin{subfigure}[b]{0.48\textwidth}
        \centering
        \includegraphics[width=\linewidth]{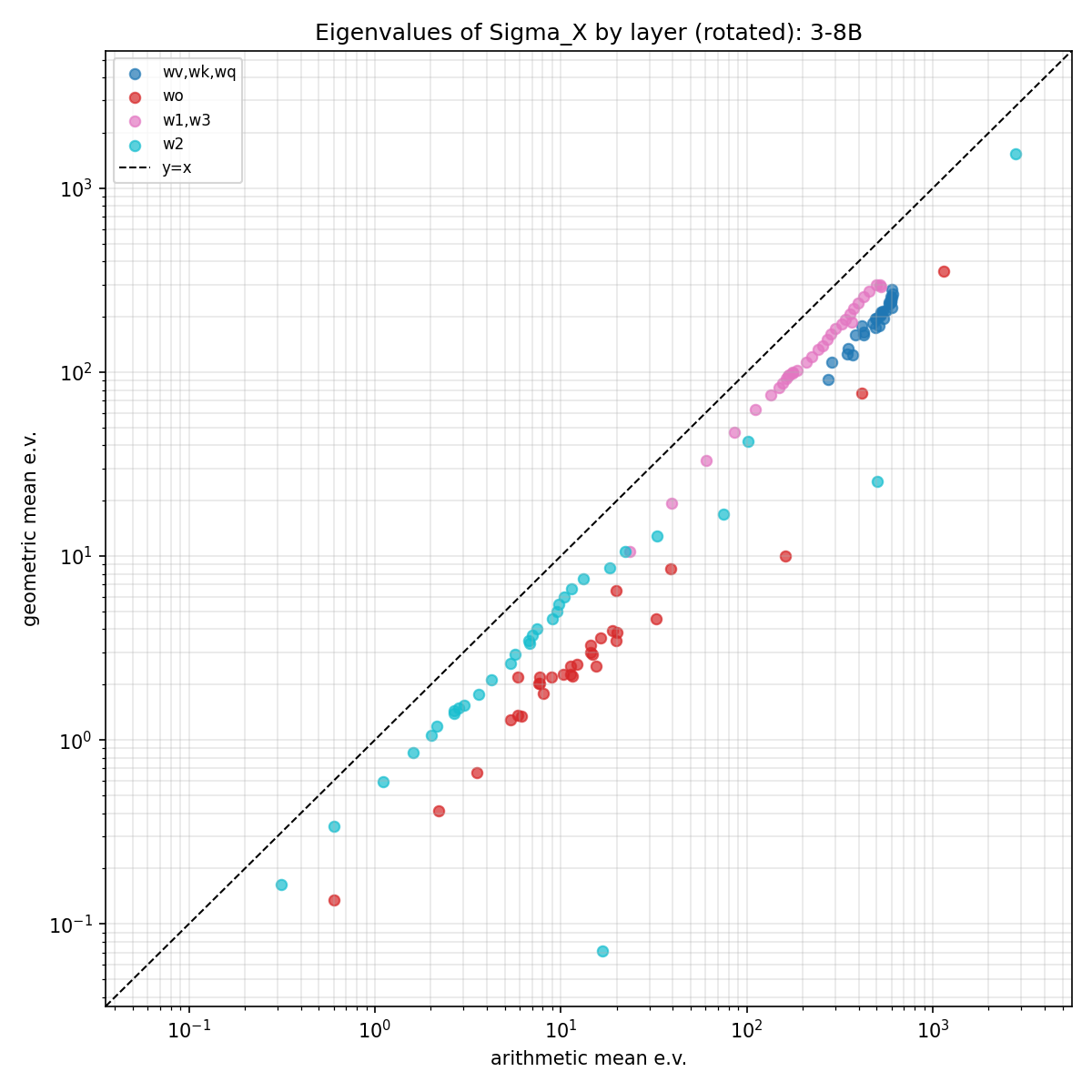}
        \caption{Scatter plot illustrating gap in AM-GM inequality for eigenvalues of $\Sigma_X$. }
    \end{subfigure}
    \hfill 
    \begin{subfigure}[b]{0.48\textwidth}
        \centering
        \includegraphics[width=\linewidth]{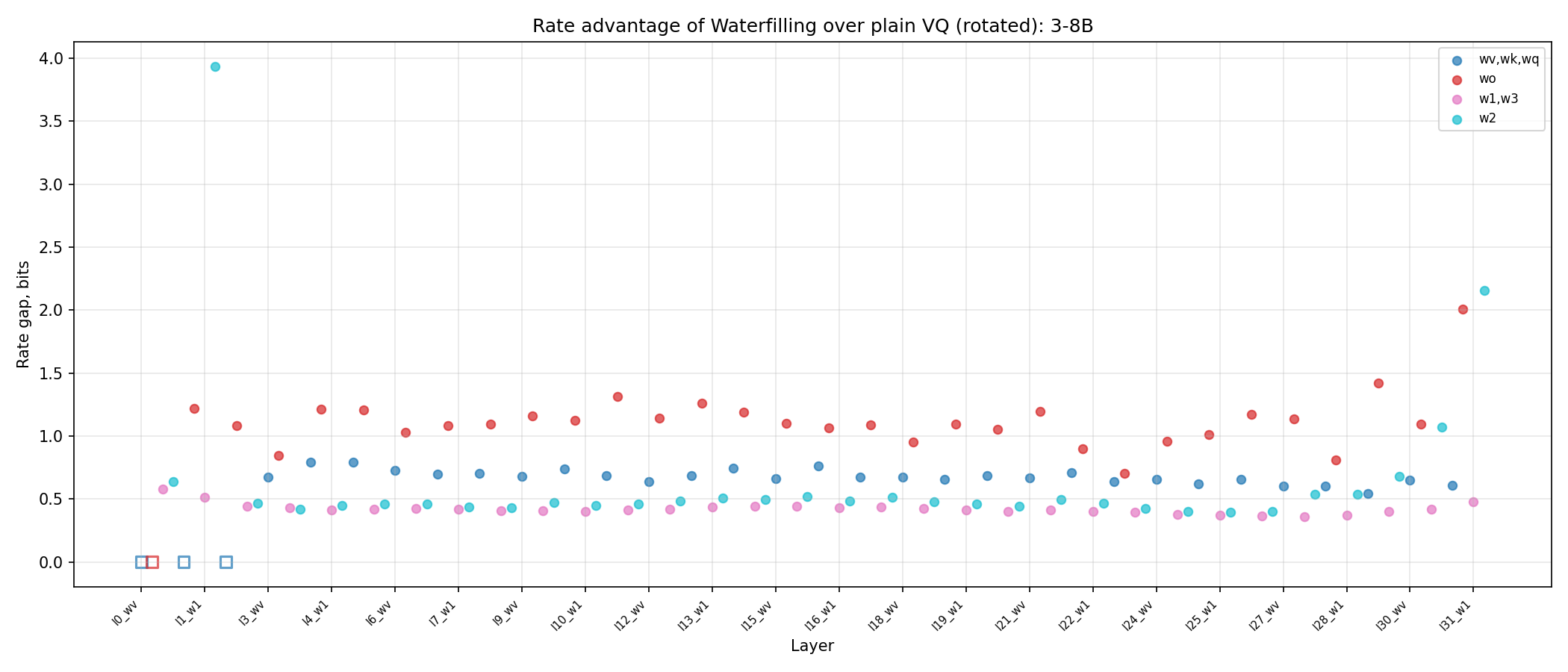}
        \caption{Rate advantage (in high-rate regime) of waterfilling over $D_{iso}$ (which
	corresponds to the case of $\Sigma_X$-oblivious encoder and decoder). 
	Squares correspond to (skipped) layers with singular $\Sigma_X$.}
    \end{subfigure}
    
    \caption{Illustrating $\Sigma_X$ of activations entering various layers of Llama-3-8B when
    processing Wikitext-2 dataset. Note that this is an estimate of the rate advantage assumes
    weight matrices are well modeled by $\mathcal{N}(0,I_n)$. In particular, actual weight
    matrices were never used for this plot.}
    \label{fig:llama_ev}
\end{figure*}

\subsection{WMSE Quantization: information theoretic limit}

The waterfilling solution derived heuristically in the previous section can be in fact made
rigorous, at least for the case of Gaussian weights $W\sim\m{N}(0,\sigma_W^2 I_n)$, which we
consider here. Another assumption we make is that the covariance matrix $\Sigma_X \in \m{PSD}_n$ (the
set of all $n\times n$ PSD matrices) is fixed and known to both encoder and decoder.

The formal problem is this: given $W\sim\m{N}(0,\sigma_W^2 I_n)$,  encoder $f:\RR^n\to[2^{nR}]$ produces a
rate-$R$ description of $W$. Subsequently, decoder $g:[2^{nR}]\to\RR^n$ converts this
description into the estimate $\hat W$. The information-theoretic question is to determine the value of the smallest distortion attainable by the best pair $(f,g)$, i.e. 
$$ D^*(\Sigma_X, R) = \min_{f,g} \frac{1}{n} \EE[ (W-\hat W)^\top \Sigma_X (W-\hat W)]\,.$$

The naive (suboptimal) solution is to use a standard isotropic Gaussian codebook of rate $R$,
which results in error covariance given by $\EE[(W-\hat W)(W-\hat W)^{\top}] = \sigma_W^22^{-2R} I_n$.
Note that the previous result is simple if we understand $\EE$ as averaging over the Gaussian
codebook, but is rather non-trivial if we want to demonstrate existence of a codebook with
nearly-white covariance matrix of errors. This was the main technical difficulty behind the
results of~\cite{ordentlich2024optimal}, cf. Theorem 13 there. Nevertheless, under this assumption
on the error distribution, we get
\begin{equation}\label{eq:Diso}
	D_{iso}(R) = 2^{-2R} {\sigma_W^2\over n} \tr \Sigma_X\,.
\end{equation}
Although this is generally very far from $D^*(\Sigma_X,R)$, such scheme works simultaneously for
all $\Sigma_X$ and does not require knowledge thereof.

Scheme that is allowed to fully exploit knowledge of $\Sigma_X$ can first compute $W' = V^\top W$,
where $V$ is the orthogonal matrix in the SVD decomposition of  $\Sigma_X = V^\top \Lambda V$.
If decoder can estimate $\hat W'$, then it can also set $\hat W = V \hat W'$. Then, the distortion
can be expressed in terms of $(W', \hat W')$ as 
$$ D = {1\over n} \sum_{i=1}^n \lambda_i \EE[(\hat W_i'-W_i')^2]\,.$$
We see that after this change of coordinates, the problem indeed becomes that of \emph{weighted}
mean-squared error (WMSE), justifying the name.

To relate distortion and rate, we consider a standard data-processing argument, see~\cite[Section 23.4]{PWbook24}:
$$ nR \ge I(W';\hat W') \ge \sum_{i=1}^n I(W'_i; \hat W_i')\,,$$
where the second inequality is due to independence of coordinates of $W'\sim
\mathcal{N}(0,\sigma_W^2 I_n)$, cf.~\cite[Theorem 6.1]{PWbook24}. Now, given value $D_i =
\EE[(\hat W_i'-W_i')^2]$ the smallest $I(W'_i; \hat W_i')={1\over 2} \log {\sigma_W^2\over D_i}$
is attained under Gaussian coupling, cf.~\cite[Section 26.1.2]{PWbook24}. Overall, we get
\begin{align*} R &\ge {1\over 2n} \sum_{i=1}^n \log {\sigma_W^2 \over D_i} \\
   D &= {1\over n} \sum_{i=1}^n \lambda_i D_i\,.
\end{align*}   
Minimizing the value of $D$ given $R$ can be done via Lagrange multipliers, resulting in the (reverse)
\emph{waterfilling solution} parameterized by $\tau>0$:
\begin{align}
D^*(\tau) &=\frac{\sigma_W^2}{n}\sum_{i=1}^n \min\{\lambda_i,\tau\},\nonumber\\
R^*(\tau) &=\frac{1}{n}\sum_{i=1}^n\frac{1}{2}\log \max\left\{1,\frac{\lambda_i}{\tau}\right\}.
\label{eq:waterfillingsolution}
\end{align}
Note that, while our argument only gives a lower bound on the minimal distortion, it can be shown
to be achievable asymptotically for large $n$, even in the single vector setting~\cite[Proposition 2.1]{HOP26}. 
Thus, in the remainder
of this paper, we will consider the value $D^*(R)$ given by waterfilling to be the actual
fundamental limit, ignoring small non-asymptotic penalty.

Just like in the previous heuristic section, we see that optimal codebook allocates rate
unequally: the PCA direction corresponding to high values of $\lambda_i$ should get quantized more
finely with $D_i \propto {1\over \lambda_i}$. In the high-rate regime we get 
\begin{align}
D_{\HR}^*(R)&=
|\Sigma_X|^{1/n}\sigma_W^2 2^{-2R}\,,    
\label{eq:Dhr}
\end{align}
and the advantage compared to uninformed isotropic coding is again given by the gap in the AM-GM
inequality, or 
$$ {1\over 2} \log {{1\over n} \tr \Sigma_X\over |\det \Sigma_X|^{1/n}} $$
bits in effective rate.

\subsection{$\Sigma_X$-oblivious decoder}\label{sec:universal_decoder}

In the setup above, we have allowed the decoder to depend on $\Sigma_X$ (in particular, to apply
SVD decomposition). However, in LLM applications this would require communicating $\Sigma_X$ (or
the SVD basis matrix $V$) to GPU's compute devices, which would be very costly. Instead, we would
like to use a very lean decoder, that only uses those bits provided with the description of $W$ to
produce its reconstruction, universally for all $\Sigma_X$.

We will therefore require that while the encoder $f(W,\Sigma_X)$ may depend on $\Sigma_X$, the decoder $g:[2^{nR}]\to \RR^n$ depends only on the bits it receives from the encoder. With this constraint, the quantization problem consists of a codebook $\m{C}\subset\RR^n$ with $|\m{C}|\leq 2^{nR}$ codewords, agreed upon by the encoder and decoder, where the encoder's job is to solve, or approximately solve, the problem
\begin{align}
\hat{W}^*=\argmin_{c\in\m{C}}[(W-c)^\top \Sigma_X (W-c)].
\label{eq:WMSEquantproblem}
\end{align}
The encoder then describes $\hat{W}^*$ using $nR$ bits, and the decoder can reconstruct $\hat{W}^*$ from those bits.

We know from the previous section that in order to achieve the informed (waterfilling) $D^*(R)$
optimal codebook needs to be anisotropic: the ``grids'' need to be denser along the PCA directions
corresponding to higher values of $\lambda_i$. In the uninformed case, we are not able to do such
adaptation (since $\Sigma_X$ is unknown at the time of the codebook design). Thus, we are forced
to use isotropic codebook. In~\cite{HOP26} a somewhat
surprising result is shown: if one generates $\m{C}$ iid from isotropic Gaussian distribution,
then simultaneously for all $\Sigma_X$ the adaptive encoder~\eqref{eq:WMSEquantproblem} achieves
the following (parametric) rate-distortion:
\begin{align*} D_{rc}(R) &= {\sigma_W^2 \over n} \sum_{i=1}^n {\lambda_i\over 1 +
\lambda_iT}\\
   R &= {1\over 2n} \sum_{i=1}^n \log (1+ \lambda_i T)\,.
\end{align*}

First, observe that in the high-rate regime $R\gg 1$ (i.e. $T\gg 1$), we get
$$ D_{rc}(R) = |\Sigma_X|^{1/n} \sigma_W^2 2^{-2R} + O(2^{-4R})\,.$$
Comparing this to~\eqref{eq:Dhr}, we observe that in the high-rate regime this scheme attains the
optimal waterfilling distortion (upto higher order corrections), \textit{despite decoder not
knowing $\Sigma_X$}. Furthermore, analysis in~\cite{HOP26} shows that universally over all
possible $\Sigma_X$ the gap between waterfilling solution $D^*(R)$ and $D_{rc}(R)$ is never more
than $0.11$ bit.

This result demonstrates that the most significant part of the gap between~\eqref{eq:Diso}
and~\eqref{eq:waterfillingsolution} is not due to suboptimal codebook design, but rather due to
suboptimal rounding to that codebook. In other words, one must focus on implementing better
rounding schemes rather than optimizing codebooks. Unfortunately, as we will see in the next
section, even for the uniform quantization ($\mathbb{Z}^n$ codebook),
solving~\eqref{eq:WMSEquantproblem} is computationally hard. Thus, despite theoretical near-optimality of
isotropic codebooks, the algorithmic difficulty makes $\Sigma_X$-oblivious quantization
very interesting and rather underexplored. A practical low-complexity algorithm, known variously as
successive-interference cancellation (SIC), Babai's algorithm, GPTQ or LDLQ, will be discussed in
the next section.

\section{Weight quantization: Practice}\label{sec:weightonly_practice}

In this section we discuss approximate solutions for the WMSE quantization problem, that can be computed efficiently. We will describe several variants of the GPTQ/LDLQ algorithms~\cite{frantar2023optq,chee2023quip}, and propose novel improvements based on our information theoretic analysis. The schemes we discuss rely on a reformulation of the ``optimal rounding'' problem~\eqref{eq:WMSEquantproblem} using the Cholesky decomposition~\cite{chen2025geometry,birnick2025lattice}. We develop this reformulation in Subsection~\ref{subsec:cholseky}. The analysis of the quantization schemes that follow requires some background in high-resolution quantization using lattices. We provide this background in Subsection~\ref{subsec:latticequantHR}, and then in Subsection~\ref{subsec:SIC} we explain and analyze the SIC rounding algorithm. In its most basic form, this algorithm is equivalent to the canonical GPTQ/LDLQ~\cite{frantar2023optq,chee2023quip}, as was recently observed in~\cite{chen2025geometry,birnick2025lattice}. In light of the information theoretic analysis above, and the analysis of the SIC algorithm, it immediately becomes clear that a simple variation of GPTQ/LDLQ called \emph{WaterSIC}, that we develop in Subsection~\ref{sec:watersic}, provides an improved distortion and is provably quite close to the lower bound~\eqref{eq:Dhr} we developed on the WMSE. The WaterSIC algorithm is further developed in~\cite{lifar2026watersic}, and is also shown to perform remarkably well for low-rate weight only quantization of LLMs.

\subsection{Reformulation of~\eqref{eq:WMSEquantproblem} via Cholesky Decomposition}
\label{subsec:cholseky}

Using the Cholesky decomposition, we can decompose $\Sigma_X$ as $\Sigma_X=U^\top U$ where $U\in\RR^{n\times n}$ is an upper triangular matrix. Plugging this into~\eqref{eq:WMSEquantproblem} we obtain that given a fixed codebook $\m{C}$ the optimal quantization problem becomes
\begin{align}
\hat{W}^*=\argmin_{c\in\m{C}}\|Y-U\cdot c\|^2,~~\text{where}~Y=UW\in\RR^n. 
\label{eq:LatDecEq}
\end{align}
Denote
\begin{align}
e_Y=Y-U\hat{W}^*,~~~e_W=W-\hat{W}^*,    
\end{align}
and note that
\begin{align}
\hat{W}^*=W+e_W=W+U^{-1}e_Y.    
\end{align}
Recall that since $W\sim\m{N}(0,\sigma^2_W I_n)$, we have that $\Pr(W\notin
\sqrt{n(1+\eps)\sigma_W^2}\m{B})\to 0$ for any $\eps>0$, as $n\to\infty$, where
$\m{B}=\{x\in\RR^n~:~\|x\|\leq 1\}$. For high-resolution quantization, where $e_Y$ is so
small that $U^{-1} e_Y$ can also be assumed small, this will imply that $\hat{W}^*$ will
typically also fall inside a ball with radius $\sqrt{n(1+\eps_R)\sigma_W^2}\m{B}$, where
$\eps_R$ vanishes with $R$ (and $n$).  This fact will be useful in the analysis below.

Before proceeding further, we note that the diagonal entries of the Cholesky matrix $U$  will play here the
same role as eigenvalues of $\Sigma_X$ did in Section~\ref{sec:weightonly_theory}. In particular,
we will see that algorithms whose codebooks have equal density in each direction (SIC/GPTQ) attain error
	\begin{equation}\label{eq:dsic_final}
		D_{\mathrm{GPTQ}}(R) \approx {2\pi e\over 12} \sigma_W^2 \left({1\over n} \sum_i U_{i,i}^2\right) 2^{-2R}\,,
\end{equation}	
which depends on \textit{arithmetic} mean of squared entries of diagonal of $U$ (which
is highly dependent on permutation of columns of $\Sigma_X$ and application of random rotations).
Our new scheme (``WaterSIC'') spends rate more judiciously (more bits to coordinates with higher
$U_{i,i}$ and attains 
	\begin{equation}\label{eq:dwatersic_final}
		D_{\mathrm{WaterSIC}}(R) \approx {2\pi e\over 12} \sigma_W^2 \left(\prod_i U_{i,i}^2\right)^{1/n}
	2^{-2R}\,.
\end{equation}	
On a first sight, we only get the same AM-GM improvement as with / without waterfilling. However,
what is remarkable and much deeper is the fact that since $|\Sigma_X| = |U|^2$, we have $(\prod_i
U_{i,i}^2)^{1/n} = |\Sigma_X|^{1/n}$, thus recovering (within factor ${2\pi e\over 12}$, which
corresponds to Koshelev's famous 0.254 bit gap~\cite[Section 24.1.5]{PWbook24}) the
information-theoretic optimal waterfilling rate $D^*_{\HR}$, see ~\eqref{eq:Dhr}.

\subsection{Preliminaries on High-resolution quantization via shaped/entropy coded lattice quantizers} 
\label{subsec:latticequantHR}

\underline{Lattices:} We review some basic lattice definitions. See~\cite{ramiBook} for a comprehensive treatment of lattices in information theory. Any lattice $L\subset \RR^n$ has a (non-unique) generating matrix $G\in\RR^{n\times n}$ such that $L=G\ZZ^n$. The covolume of the lattice $L$, denoted $\mathrm{covol}(L)$) is defined as $| G|$. The point density of a lattice is $\gamma(L)=\covol^{-1}(L)=|G|^{-1}$. For a lattice $L\subset\RR^n$ we define the nearest neighbor quantizer $Q_L:\RR^n\to L$ as
\begin{align}
Q_{L}(x)=\argmin_{\lambda\in L}\|x-\lambda\|, 
\end{align}
where ties are broken arbitrarily, but in systematic manner. The Voronoi region $\m{V}_L$ is defined as the set of all points in $\RR^n$ that are closer to $0$ than to any other lattice point
\begin{align}
\m{V}_L=\left\{x\in\RR^n~:~Q_L(x)=0\right\}.    
\end{align}
The volume of the Voronoi region (as well as any other fundamental cell of $L$) is $\mathrm{Vol}(\m{V}_L)=\covol(L)=|G|$.

We define the second moment of the lattice $L$ as
\begin{align}
\sigma^2(L)=\frac{1}{n}\EE\|Z\|^2,    
\end{align}
where $Z\sim\Unif(\m{V}_L)$ is a random vector uniformly distributed over the Voronoi region of $L$. 

\medskip

\underline{High-resolution lattice quantizers:} Designing a quantizer for the WMSE problem with $\Sigma_X$-oblivious decoder entails choosing the
codebook $\m{C}$, which consists of $2^{nR}$ vectors in $\RR^n$. One way to construct such a
codebook is to start with an infinite constellation, in our case a lattice $L\subset\RR^n$, and
then choose from $L$ only $2^{nR}$ points~\cite{ramiBook}. This can be done by choosing a shaping
region $\m{S}\subset\RR^n$, e.g. $\m{S}=r\m{B}$ for some $r>0$, and taking $\m{C}=L\cap\m{S}$
where $\m{S}$ is dilated such that $|L\cap \m{S}|=2^{nR}$. Correspondingly, this way of thinking
about vector quantization splits the problem into design of good infinite constellations (with
prescribed point density) and shaping regions that chop off a finite part of it.

Many practical solutions for weight-only quantization that were considered in the literature fall
under the shaped lattice quantization paradigm: GPTQ with INT
constellations~\cite{frantar2023optq} as well as LDLQ with $E_8$-based codebooks as in
QuIP\#~\cite{chee2023quip,tseng2024quip}, and many more.

Assume the point $\hat{W}^*\in L$ that minimizes~\eqref{eq:LatDecEq} with respect to the entire
lattice $L$ (rather than $L\cap \m{S}$) is within the shaping region $\m{S}$. In this case,
$e_Y\in\m{V}_{UL}$ and if we further assume $e_Y\sim\Unif(\m{V}_{UL})$ the distortion is
$\sigma^2(UL)$. However, whenever the
\emph{overload} event $\hat{W}^*\notin\m{S}$ occurs, the squared error may significantly exceed
$\sigma^2(UL)$. In high-resolution quantization, we may assume $\hat{W}^*\approx W$, and the
overload probability is well approximated by $\Pr(W\notin\m{S})$. As mentioned above, taking
$\m{S}=\sqrt{n(1+\eps)\sigma_W^2}\m{B}$, with some small $\eps>0$ (or any other $\m{S}$ that
contains this ball) suffices for achieving vanishing overload probability. Thus, for such choice
of $\m{S}$ we have that $D\approx \sigma^2(UL)$.

If $L$ is ``sufficiently dense'' with respect to $\m{S}$, it holds that $|L\cap \m{S}|\approx \gamma(L)\cdot\Vol(S)$ where $\gamma(L)$ is the point density of $L$ (see~\cite[Lemma 3.3]{ordentlich2023bounds} for precise bounds on $|L\cap \m{S}|$). This approximation becomes accurate in the high-resolution limit (as high-resolution corresponds to large $\gamma(L)$).  
We can therefore approximate the rate as 
\begin{align}
R=\frac{1}{n}\log|L\cap \m{S}|\approx \frac{1}{n}\log(\Vol(\m{S}))+\frac{1}{n}\log \gamma(L).    
\label{eq:RvsDensity}
\end{align} 
It follows that in the high-resolution regime the quantization rate $R$ and the normalized log
point-density $\frac{1}{n}\log\gamma(L)$ are equal up to a constant. Furthermore, recalling that
$D\approx\sigma^2(UL)$ provided that $\Pr(W\notin\m{S})$ is sufficiently small, we can express $D$
in the form familiar to us. Specifically, we get from~\eqref{eq:RvsDensity} 
\begin{align}
&D\approx \left(\sigma^2(UL)\cdot \gamma^{\frac{2}{n}}(UL) \right)|U|^{\frac{2}{n}}\cdot \Vol^{\frac{2}{n}}(\m{S})\cdot 2^{-2R}\nonumber\\
&=\left(\sigma^2(UL)\cdot \gamma^{\frac{2}{n}}(UL) \right) |\Sigma_X|^{\frac{1}{n}}\cdot
\Vol^{\frac{2}{n}}(\m{S})\cdot 2^{-2R}\,,
\label{eq:ShapedLatRD}
\end{align}
where we also used the fact $\gamma(UL) = |U|^{-1} \gamma(L)$. The term in $(\cdot)$ is a famous
and extremely well-studied quantity~\cite{ramiBook} known as the normalized second moment (NSM) of $UL$, defined as $\sigma^2(UL)\cdot
\gamma^{\frac{2}{n}}(UL)$. From~\eqref{eq:ShapedLatRD} we see that the tradeoff between rate and distortion is
determined by: 1)The volume of the shaping region (which is required to satisfy
$\Pr(W\notin\m{S})\ll 1$);  2)The NSM of $UL$.

We stress that $\sigma^2(UL)$ is the WMSE distortion when the $\argmin_{c\in L}$ in~\eqref{eq:LatDecEq} is solved exactly. As we discuss below, finding the exact solution to this problem is generally infeasible for large $n$, and one must resort to approximate solutions. In the case where a sub-optimal quantizer $q:\RR^n\to L$ is used, the NSM is replaced with $\frac{1}{n}\EE\|Y-U\cdot q(Y)\|^2\cdot \gamma^{\frac{2}{n}}(UL)$ in~\eqref{eq:ShapedLatRD}.

An alternative approach for shaping is entropy coding (EC). If quantization with variable rate is allowed, and only the \emph{expected} quantization rate is constrained to be at most $R$, one can quantize the source to the point $\hat{W}^*\in L$ that minimizes~\eqref{eq:LatDecEq} (assuming $\m{C}=L$), and then describe the resulting point in bits using EC.

A sequence of classic works~\cite{bennett1948spectra,panter1951quantization,zador1982asymptotic,gersho79} have considered the case of high-resolution entropy coded quantization with infinite constellation, and subsequent work restricted attention to the case where the infinite constellation is taken as a lattice ~\cite{eyuboglu1993lattice,zamir1996lattice,ramiBook}. Those works considered the MSE case with $\Sigma_X=I_n$, but their conclusion adapted to the WMSE case is that for high-resolution quantization and large $n$ the tradeoff between distortion and average quantization rate also obeys~\eqref{eq:ShapedLatRD} with $\Vol^{\frac{2}{n}}(\m{S})$ replaced by $2^{2h(W)}=2\pi e\sigma_W^2$, where $h(\cdot)$ denotes differential entropy. Namely, for entropy coded high-resolution lattice quantization
\begin{align}
D\approx \left(\sigma^2(UL)\cdot \gamma^{\frac{2}{n}}(UL)\right) 2\pi e\cdot|\Sigma_X|^{\frac{1}{n}}\sigma_W^2\cdot 2^{-2R}.
\label{eq:ecdqRD}
\end{align}

Since rate and normalized lattice point density are equivalent under either shaped or entropy coded lattice quantization, it suffices to consider the tradeoff between WMSE distortion and $\gamma(L)$.

To that end, we now provide a lower bound on $\sigma^2(UL)$ that depends only on $|U|$ and on $\gamma(L)$. This bound is due to Zador~\cite{zador1982asymptotic} (See also~\cite[eq. (82)]{ConwaySloane}). It follows from the fact that $\sigma^2(UL)$ is the power of a random vector uniformly distributed over $\m{V}_{UL}\subset\RR^n$, which is a convex body of volume $V=|U|\cdot\gamma^{-1}(L)$. Among all bodies in $\RR^n$ of volume $V$, the power of a uniform random vector is minimized when the body is a $\ell_2$ ball.\footnote{In fact, it is known~\cite{ordentlich2025voronoi} that for almost all lattices (with respect to the natural measure on the space of lattices) $\sigma^2(L)$ is  only $(1+O(1/n))$ greater than the second moment of the corresponding $\ell_2$ ball. Thus, this lower bound is asymptotically attained by a ``typical'' lattice.}

\begin{proposition}[Zador]
For any full-rank lattice $L\subset\RR^n$ and any full-rank matrix $U\in\RR^{n\times n}$
\begin{align}
\sigma^2(UL)\geq \frac{\Gamma\left(\frac{n}{2}+1\right)}{(n+2)\pi}\gamma(L)^{-\frac{2}{n}}|U|^{\frac{2}{n}}\approx\frac{\gamma(L)^{-\frac{2}{n}}|U|^{\frac{2}{n}}}{2\pi e}.
\label{eq:DlbZador}
\end{align}
\label{lem:DlbZador}
\end{proposition}
In~\eqref{eq:DlbZador}, $\Gamma(\cdot)$ is the gamma function. The last approximation can be replaced with a lower bound at the expense of multiplying the right-hand side term by a factor of $\frac{n}{n+2}$. Note that substituting~\eqref{eq:DlbZador} in~\eqref{eq:ecdqRD}, gives $D\gtrsim D_{\HR}^*(R)$, where $D_{\HR}^*(R)$ is given in~\eqref{eq:Dhr}. 

Even without relying on the validity of~\eqref{eq:ecdqRD}, Proposition~\ref{lem:DlbZador} provides a simple lower bound on the WMSE distortion of any high-resolution quantization scheme for which the reconstruction satisfies $\hat{W}\in L$, for a lattice $L\subset\RR^n$. This follows since for any such quantization scheme we have (due to~\eqref{eq:LatDecEq}) 
\begin{align}
D\geq \frac{1}{n}\min_{c\in L}\EE\|Y-U\cdot c\|^2=\frac{1}{n}\|Y-Q_{UL}(Y)\|^2, \nonumber   
\end{align}
where $Q_{UL}:\RR^n\to UL$ is the nearest neighbor (optimal) quantizer for the lattice $UL\subset\RR^n$. Under the high-resolution assumption $e_Y=Y-Q_{UL}(Y)\sim\Unif(\m{V}_{UL})$ and consequently 
\begin{align}
D\geq \sigma^2(UL)\gtrsim \frac{\gamma(L)^{-\frac{2}{n}}|U|^{\frac{2}{n}}}{2\pi e} = |\Sigma_X|^{\frac{1}{n}} \frac{\gamma(L)^{-\frac{2}{n}}}{2\pi e},
\label{eq:WMSElatLB}
\end{align}
where we have used Proposition~\ref{lem:DlbZador} and the fact that $|U|^2=|\Sigma_X|$.


For example, in GPTQ/LDLQ it is common to take $L=\alpha\ZZ^n$ as the base lattice, whose density is $\gamma(L)=\alpha^{-n}$. From~\eqref{eq:WMSElatLB} we see that its attained distortion must satisfy $D\gtrsim \alpha^2 \frac{|\Sigma_X|^{1/n}}{2\pi e}$. 

Note that~\eqref{eq:WMSElatLB} relies on the assumption $e=UW-Q_{UL}(UW)\sim\Unif(\m{V}_{UL})$. This holds asymptotically in the limit of high resolution ($\gamma(L)\to\infty$) for all full-rank matrices $U\in\RR^{n\times n}$. If one uses  dithered lattice quantization~\cite{ramiBook}, this assumption holds exactly for all $\gamma(L)$. However, in the case of dithered lattice quantization the resulting estimate $\hat{W}^*$ would not be in $L$. Instead, under dithered lattice quantization we have $\hat{W}^*=Y+e$ where $e\sim\Unif(\m{V}_{UL})$, $e\indep Y$. Consequently, one can benefit by setting $\hat{W}=\beta\hat{W}^*$ with $\beta<1$ as the estimate for $W$, as in this case
\begin{align}
\frac{1}{n}\EE(X^T&(W-\hat{W}))^2=\EE((1-\beta)X^\top W-\beta e)^2\nonumber\\
&=(1-\beta)^2 \sigma_W^2\frac{\trace(\Sigma_X)}{n}+\beta^2 \sigma^2(UL).    
\end{align}
This is the same shrinkage effect that we discussed in~\cite[Section I.b]{OP26part1}
Note that $\beta\to 1$ as $\sigma^2(UL)$ decreases, which is the case for high-resolution
quantization. 


\medskip

\subsection{Product Codebooks, Codebook Spacing, and Successive Cancellation}
\label{subsec:SIC}

This section discusses efficient quantization schemes via successive interference cancellation (SIC). The discussion is not restricted to lattice quantizers, and we will see that SIC can be applied for any product code. Afterwards, we will specialize the discussion to lattice quantizers.

In Section~\ref{sec:weightonly_theory} we argued that the relevant setup for LLMs is the uninformed/$\Sigma_X$-oblivious decoder case. We introduced there the problem of quantizing a single column of the weight matrix. In practice, however, a weight matrix will typically have $a\gg 1$ column vectors, and the WMSE matrix $\Sigma_X$ is common to all of them (as they all operate on the same activations). If the number of column vectors is of the same order as $n$, it is possible for the decoder to send an $O(n)$ bits description of $\Sigma_X$ to the uninformed decoder with only a negligible effect on the quantization rate. In particular, the encoder can apply a diagonal  matrix $\m{A}=\diag(\alpha_1,\ldots,\alpha_n)\in\RR^{n\times n}$, whose spacing coefficients $\{\alpha_i\}$ are chosen based on $\Sigma_X$, and quantize $W$ using the codebook $\m{A}\cdot\m{C}$ rather than the codebook $\m{C}$. The cost of reporting $(\alpha_1,\ldots,\alpha_n)$ to the decoder, is amortized over the $a$ columns of the weight matrix, which makes it negligible provided that $a$ is not too small. In fact, spacing the codebook $\m{C}$ used for the reconstruction of $W$ to the codebook $\m{A}\cdot\m{C}$ is equivalent to changing $X^\top$ to $X^\top \m{A}$. Consequently, in LLMs it is sometimes the case that $\m{A}$ can be absorbed in the activations using the layer-norm with no additional cost associated with describing the scales~\cite{xiao2023smoothquant}.

Our focus is on obtaining efficient approximate solution to the optimization in equation~\eqref{eq:WMSEquantproblem}, which becomes equivalent to the optimization in~\eqref{eq:LatDecEq} when $\Sigma_X=U^\top U$ is decomposed using the Cholesky decomposition. Assume the codebook $\m{C}\subset\RR^n$ is a \emph{product codebook} of the form 
\begin{align}
\m{C}=\m{C}_1\times\cdots\times\m{C}_n,~~\text{where}~\m{C}_i\subset\RR~ \forall i\in[n].
\end{align}
Product codebooks are a common practical choice as it allows for fast (and parallel) decoding. Under the MSE criterion (corresponding to $\Sigma_X=I_n$) they also result in fast optimal encoding, as in this case $U=I_n$ and~\eqref{eq:LatDecEq} becomes
\begin{align}
\hat{W}_{\text{MSE}}&=\argmin_{c_1\in\m{C}_1,\ldots,c_n\in\m{C}_n} \sum_{i=1}^n (W_i-c_i)^2\nonumber\\
~~\Longrightarrow \hat{W}_{\text{MSE},i}&=\argmin_{c_i\in\m{C}_i}(W_i-c_i)^2,~~~\forall i\in[n].   
\end{align}
Under the WMSE criterion, corresponding to non-diagonal $\Sigma_X$, product codebooks do not in general lend
themselves to fast optimal encoding algorithms. In particular, if
$\m{C}=\ZZ^n=\ZZ\times\cdots\times\ZZ$, then exactly solving the optimization
in~\eqref{eq:LatDecEq} requires solving the closest vector problem (CVP) in the lattice
$\tilde{L}=U\ZZ^n$, as observed in~\cite{chen2025geometry,birnick2025lattice}. This problem is
known to be NP-hard. Consequently, one must resort to sub-optimal algorithms.

Here, we restrict attention to successive
interference cancellation (SIC). The most general form of this algorithm is provided in Algorithm~\ref{alg:genSIC} and illustrated in Figure~\ref{fig:SICffedfe}. In \emph{generalSIC} the $n$ codebooks $\m{C}_1,\ldots,\m{C}_n\subset \RR$ whose product is $\m{C}\subset\RR^n$ can be arbitrary, and each of them is further scaled by the corresponding $\alpha_i$. This scaling can in general be absorbed in the codebooks definition. However, since we allow $\m{A}$ to depend on $\Sigma_X$ through the matrix $U\in\RR^{n\times n}$ while the codebooks $\m{C}_1,\ldots,\m{C}_n$ are not allowed to depend on $\Sigma_X$, we do not absorb $\m{A}$ into the codebooks.
In generalSIC, we first filter the vector $Y$ using the feedforward filter $F\in\RR^{n \times n}$. Any matrix $F\in\RR^{n\times n}$ can be used here. Then, the vector $\hat{W}$ is generated sequentially, starting from $\hat{W}_n$ up to $\hat{W}_1$. At every step $i$ the scalar $(FY)_i$ is fed to the quantizer to generate $\hat{W}_i$, and then the vector $FY$ is updated to $FY-\hat{W}_i B_{:,i}$. The feedback matrix $B\in\RR^{n\times n}$ is strictly lower triangular, due to causality.
We have that $\hat{W}=\m{A} Q\left(\m{A}^{-1}(FY-B\hat{W}) \right)$. If the quantizer is modeled as adding independent quantization noise $Z$, we therefore obtain \begin{align}
\hat{W}=(I_n+B)^{-1}FY+(I_n+B)^{-1}\m{A}Z.   
\end{align}
Finally, the estimate $\hat{W}$ is further scaled by $\beta>0$.

Let us temporarily fix the spacing matrix $\m{A}$. The choices of $F$ and $B$ and $\beta$, should strike a balance between two quantities. First, we want the $\ell_2$ norm of
\begin{align}
&e_Y=Y-\beta U\hat{W}\nonumber\\
&=\left(I_n-\beta U(I_n+B)^{-1}F\right)Y-\beta U(I_n+B)^{-1}\m{A} Z
\end{align}
to be as small as possible. On the other hand, we also want $\hat{W}$ to have differential entropy as small as possible, as this will dictate the volume of the shaping region required for capturing it (or equivalently, the average rate of encoding the quantizer's output using entropy coding).

However, in the limit of large quantization rate, where the energy of $Z$ vanishes, the optimal choice of $F,B,\beta$ tends to the solution for which $\left(I_n-\beta U(I_n+B)^{-1}F\right)=0$. This corresponds to choosing
\begin{align}
&\beta_{\HR}=1,~F_{\HR}=(\diag(U))^{-1},\nonumber\\
&~~~~~~~~B_{\HR}=(\diag(U))^{-1}U-I_n.
\label{eq:HRfilters}
\end{align}
If we use the generalSIC algorithm with $F,B,\beta$ from~\eqref{eq:HRfilters}, and use the same scalar codebook $\m{C}_0$ for quantizing all coordinates, the algorithm simplifies to Algorithm~\ref{alg:SIC} which we simply call the SIC algorithm.

As we shall see below, the choice of $\m{A}$ has an important impact on the scheme's performance.
In the special case where $\m{A}=\alpha I_n$ for some $\alpha>0$, the SIC algorithm becomes
completely equivalent to the canonical GPTQ algorithm (which in turn, is equivalent to the LDLQ
algorithm~\cite{chee2023quip}). This equivalence was recently shown
in~\cite{chen2025geometry,birnick2025lattice}. Thus, in the sequel, we refer to SIC with
$\m{A}=\alpha I_n$ as the GPTQ algorithm. Note however that the GPTQ algorithm was originally
presented in~\cite{frantar2023optq} from a \emph{noise-shaping} point of view, where the quantization
noise is filtered and fed to the next quantizer. In the SIC point of view, it is the signal $Y$, rather than the quantization noise, that is filtered and fed to the next quantizer. 

We find the SIC point of view of~\cite{chen2025geometry,birnick2025lattice} more intuitive than the original noise-filtering perspective. Furthermore, under the SIC point of view, the problem is also similar to V-BLAST decoding for the Gaussian MIMO channel~\cite{wolniansky1998v}.


\begin{algorithm*}[h!]
\caption{generalSIC}
\label{alg:genSIC}
\begin{algorithmic}
\State \textbf{Inputs:} $Y\in\RR^n$, feed-forward filter $F\in\RR^{n\times n}$, strictly upper triangular feedback filter $B\in\RR^{n\times n}$, diagonal spacing matrix $\m{A}=\diag(\alpha_1,\ldots,
\alpha_n)\in\RR_+^{n\times n}$, scaling coefficient $\beta>0$ and codebooks $\m{C}_1,\ldots,\m{C}_n\subset\RR$
\State \textbf{Outputs:} $\hat{W}\in \m{C}_1\times\cdots\times \m{C}_n$.
\vspace{2mm}
\State $Y\gets FY$
\For{$i=n:1$}
\State $\hat{W}_i\gets \alpha_iQ_i\left(\frac{Y_{i}}{\alpha_i}\right)$\Comment{$Q_i(x)=\argmin_{c_i\in\m{C}_i} (x-c_i)^2$}
\State $Y\gets Y-\hat{W}_i\cdot B_{:,i}$\Comment{$B_{:,i}$ is the $i$th column of $B$}
\EndFor
\State $\hat{W}\gets \beta \hat{W}$
\end{algorithmic}
\end{algorithm*}

\begin{algorithm*}[h!]
\caption{SIC}
\label{alg:SIC}
\begin{algorithmic}
\State \textbf{Inputs:} $Y\in\RR^n$, upper triangular $U\in\RR^{n\times n}$,  diagonal matrix $\m{A}=\diag(\alpha_1,\ldots,
\alpha_n)\in\RR_+^{n\times n}$ and scalar codebook $\m{C}_0\subset\RR$
\State \textbf{Outputs:} $c_{\mathrm{SIC}}\in \m{C}_0^{\otimes n}$
\vspace{2mm}
\For{$i=n:1$}
\State $c_{\mathrm{SIC},i}\gets \alpha_i Q\left(\frac{Y_{i}}{\alpha_i U_{i,i}}\right))$\Comment{$Q(x)=\argmin_{c\in\m{C}_0} (x-c)^2$}
\State $Y\gets Y-c_{\mathrm{SIC},i}\cdot U_{:,i}$\Comment{$U_{:,i}$ is the $i$th column of $U$}
\EndFor
\end{algorithmic}
\end{algorithm*}

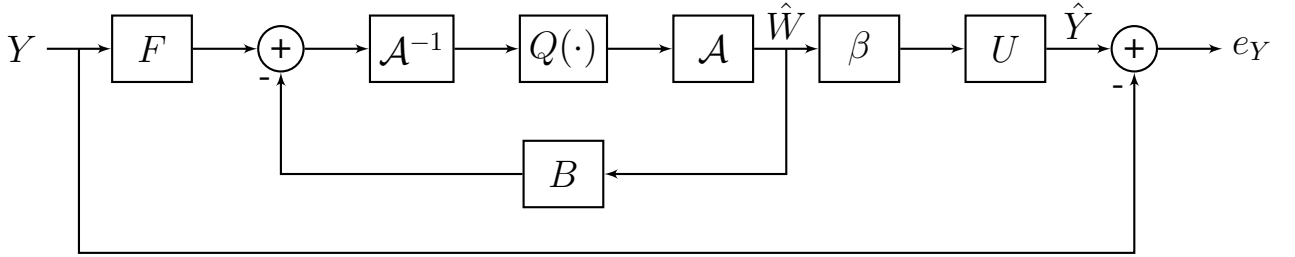
\begin{figure*}
\centering

\begin{tikzpicture}[auto, >=latex', node distance=0.85cm, thick]

    \tikzset{block/.style = {draw, rectangle, minimum height=2.5em, minimum width=3em, align=center},
             sum/.style = {draw, circle, inner sep=0mm, minimum size=6mm}}


    \node (Y) at (0,0) {\Large $Y$};

    \node [block, right=of Y] (F) {\Large $F$};
    \node [sum, right=of F] (sum1) {\Large +};

    \node [block, right=of sum1] (A) {\Large $\m{A}^{-1}$};
    \node [block, right=of A] (Qnew) {\Large $Q(\cdot)$};
    \node [block, right=of Qnew] (Ainv) {\Large $\m{A}$};

    \node [block, right=of Ainv] (Beta) {\Large $\beta$};

    \node [block, right=of Beta] (U) {\Large $U$};
    \node [sum, right=of U] (sum2) {\Large +};

    \node [right=of sum2] (ey) {\Large $e_Y$};

    \node [block, below=0.75cm of Qnew] (B) {\Large $B$};


    \draw [->] (Y) -- (F);
    \draw [->] (F) -- (sum1);
    \draw [->] (sum1) -- (A);
    \draw [->] (A) -- (Qnew);
    \draw [->] (Qnew) -- (Ainv);

    \draw [->] (Ainv) -- node [midway, above] {\Large $\hat{W}$} (Beta);
    \draw [->] (Beta) -- (U);

    \draw [->] (U) -- node [midway, above] {\Large $\hat{Y}$} (sum2);
    \draw [->] (sum2) -- (ey);


    \draw [->] ($(Ainv.east)!0.5!(Beta.west)$) |- (B);
    \draw [->] (B) -| node[pos=0.95, left] {\Large -} (sum1);

    \draw [->] ($(Y.east)!0.5!(F.west)$) |- ++(0,-2.7) -| node[pos=0.95, left] {\Large -} (sum2);

\end{tikzpicture}

   \caption{Illustration of the generalSIC quantization algorithm. The matrix $F\in\RR^{n\times n}$ is unrestricted, whereas the matrix $B\in\RR^{n\times n}$ must be strictly upper triangular due to causality. The matrix $\m{A}\in\RR^{n\times n}$ is positive diagonal, and determines the spacing of the quantizer in each coordinate. The parameter $\beta>0$ is a scaling parameter. The GPTQ algorithm is obtained as a special case for $B,F,\beta$ taken as in~\eqref{eq:HRfilters}, and $\m{A}=\alpha I_n$. The WaterSIC algorithm is obtained with the same $F,B,\beta$, but with $\m{A}=\diag(\alpha_1^{\text{Water}},\ldots,\alpha_n^{\text{Water}})$ and $\alpha_i^{\text{Water}}$ are given in~\eqref{eq:alphaAMGM}}
  \label{fig:SICffedfe}
\end{figure*}

\medskip

\medskip


In practice, the sub-codebook $\m{C}_0$ is often taken as the constellations corresponding to data-type FP8, FP4, INT8, etc., resulting in $\hat{W}$ that can be used in fast MatMul hardware. For the remainder of our discussion we will assume $\m{C}_0=\ZZ$, such that the product codebook is a lattice. In this case, the SIC algorithm is merely a low-complexity sub-optimal  algorithm for the closest vector problem, which is often also referred to as Babai's nearest plane
algorithm~\cite{babai86}. For small $n$ the sphere-decoder~\cite{fincke1985improved,agrell2002closest} can be used, and for large $n$ other sub-optimal algorithms can equally be used instead of SIC.

We now analyze the distortion attained by the SIC algorithm, with $\m{C}_0=\ZZ$. The resulting $\hat{W}$ is in the lattice $\m{A}\ZZ^n$, whose cardinality is unbounded. The description of $\hat{W}$ in bits will be handled via shaping or entropy coding, and for now we only constrain the point density of $L=(\alpha_1\ZZ)\times\cdots\times(\alpha_n \ZZ)$ to $\gamma(L)^{-1}=\left(\prod_{i=1}^n\alpha_i\right)=\alpha^n$. Recall that from Proposition~\ref{lem:DlbZador} and~\eqref{eq:WMSElatLB} we have that even if the optimal solution to~\eqref{eq:LatDecEq} is found, the distortion must satisfy $D\gtrsim\alpha^2 \frac{|\Sigma_X|^{2/n}}{2\pi e}$. Note that for $\m{C}_0=\ZZ$, the quantizer $Q$ used in the SIC algorithm takes the simple form $Q(x)=\mathrm{round}(x)$ so that the execution of the SIC algorithm is particularly simple. The next result can be deduced from~\cite{babai86}, and we bring the simple proof for completeness.
\begin{lemma}
Assume we apply the SIC Algorithm with $y\in\RR^n$, upper triangular $U\in\RR^{n\times n}$, and $\m{C}_i=\alpha_i\ZZ$ for all $i\in[n]$. Then
\begin{align}
 e_{\mathrm{SIC}}=y-U\cdot c_{\mathrm{SIC}} \in \m{P}_{\{U_{i,i},\alpha_i\}_{i=1}^n},\nonumber
\end{align}
where
\begin{align}
\m{P}_{\{U_{i,i},\alpha_i\}_{i=1}^n}=\prod_{i=1}^n \left[-\frac{|\alpha_i U_{i,i}|}{2},\frac{|\alpha_i U_{i,i}|}{2} \right). \nonumber
\end{align}
If in addition $e_{\mathrm{SIC}}\sim\Unif(\m{P}_{\{U_{i,i},\alpha_i\}_{i=1}^n})$ we have that
\begin{align}
D_{\mathrm{SIC}}&\triangleq\frac{1}{n}\EE\|Y-U\cdot c_{\mathrm{SIC}}\|^2\nonumber\\
&=\frac{1}{12}\cdot\frac{1}{n}\sum_{i=1}^n (\alpha_i U_{i,i})^2.  
\label{eq:DsicError}
\end{align}
\label{lem:FundCellSIC}
\end{lemma}

\begin{proof}
Let $c_{\mathrm{SIC}}(y,U)$ be the result of applying the SIC algorithm with inputs $y\in\RR^n$ and $U\in\RR^{n\times n}$. Denote $\m{A}=\diag(\alpha_1,\ldots,\alpha_n)\in\RR^{n\times n}$. The lemma follows from combining the two following observations:
\begin{enumerate}
\item $\m{P}_{\{U_{i,i},\alpha_i\}_{i=1}^n}=\left\{y\in\RR^n~:~c_{\mathrm{SIC}}(y,U)=0 \right\}$;
\item For any $z\in\ZZ^n$ it holds that $$c_{\mathrm{SIC}}(y+U \m{A} z,U)=\m{A}z+c_{\mathrm{SIC}}(y,U).$$
\end{enumerate}
Thus the SIC algorithm induces the partition of space to decision regions 
\begin{align}
\m{D}_{z}=U\m{A} z+\m{P}_{\{U_{i,i},\alpha_i\}_{i=1}^n},~\forall z\in\ZZ^n    
\end{align}
and $c_{\mathrm{SIC}}(y,U)=\m{A}z$ iff $y\in\m{D}_z$. Consequently, $e_{\mathrm{SIC}}=y-U\cdot c_{\mathrm{SIC}}(y,U)\in \m{P}_{\{U_{i,i},\alpha_i\}_{i=1}^n} $. The claim on $D_{\mathrm{SIC}}$ immediately follows from the product structure of $\m{P}_{\{U_{i,i},\alpha_i\}_{i=1}^n}$.
\end{proof}

It follows from~\eqref{eq:DsicError} that the distortion for the GPTQ algorithm followed by entropy coding is as given in~\eqref{eq:dsic_final}.

By the arithmetic-mean geometric-mean inequality (AM-GM) and~\eqref{eq:DsicError} we have
\begin{align}
&D_{\mathrm{SIC}}=\frac{1}{12}\cdot\frac{1}{n}\sum_{i=1}^n (\alpha_i U_{i,i})^2\nonumber\\
&\geq \frac{1}{12}\left(\prod_{i=1}^n{\alpha_i}\right)^{\frac{2}{n}}\left(\prod_{i=1}^n U_{i,i} \right)^{\frac{2}{n}} =\frac{\alpha^2}{12}|\Sigma_X|^{\frac{2}{n}}.
\label{eq:DsicBound}
\end{align}
The inequality above is achieved with equality iff the scaling factors are of the form
\begin{align}
\alpha_i^{\text{Water}}\triangleq \alpha\cdot \frac{|U|^{\frac{1}{n}}}{|U_{i,i}|},~~\forall i\in[n].
\label{eq:alphaAMGM}
\end{align}

\subsection{WaterSIC algorithm}
\label{sec:watersic}

Motivated by~\eqref{eq:DsicBound} and~\eqref{eq:alphaAMGM}, we propose Algorithm~\ref{alg:AMGM}: the WaterSIC weight-only quantization algorithm for a weight matrix $W\in\RR^{n\times a}$. The properties of the resulting reconstruction $\hat{W}$ are provided in Lemma~\ref{lem:SICAMGM}.

\begin{algorithm*}[h!]
\caption{WaterSIC weight-only quantization}
\label{alg:AMGM}
\begin{algorithmic}
\State \textbf{Inputs:} $W\in\RR^{n\times a}$, PSD matrix $\Sigma_X$ and point density $\alpha>0$.
\State \textbf{Outputs:} $Z_{\mathrm{SIC}}\in \ZZ^{n\times a}$ and $(\alpha_1,\ldots,\alpha_n)\in\RR_+^n$ such that $\hat{W}=\diag(\alpha_1,\ldots,\alpha_n) Z_{\mathrm{SIC}}$.
\vspace{2mm}
\State Compute upper-triangular $U\in\RR^{n\times n}$ such that $\Sigma_X=U^\top U$\Comment{Using the Cholesky decomposition}
\State $Y\gets U W$
\State $\alpha_i\gets\alpha\cdot \frac{|U|^{\frac{1}{n}}}{|U_{i,i}|},~~\forall i\in[n]$
\State $Z_{\mathrm{SIC}}\gets 0^{n\times a}$\Comment{Initialize $Z_{\mathrm{SIC}}$ with zeros}
\For{$i=n:1$}
\State $Z_{\mathrm{SIC},i,:}\gets \mathrm{round}\left(\frac{Y_{i,:}}{\alpha_i U_{i,i}}\right)$ \Comment{$Y_{i,:}$ and $Z_{\mathrm{SIC},i,:}$ is the $i$th row of $Y$ and $Z_{\mathrm{SIC}}$, respectively}
\State $W\gets Y-\alpha_iU_{:,i}\cdot Z_{\mathrm{SIC},i,:}$\Comment{$U_{:,i}$ is the $i$th column of $U$}
\EndFor
\end{algorithmic}
\end{algorithm*}

\begin{lemma}
\label{lem:SICAMGM}
The reconstruction $\hat{W}=\diag(\alpha_1,\ldots,\alpha_n) Z_{\mathrm{SIC}}$ produced by the WaterSIC weight-only quantization algorithm satisfies: 
\begin{enumerate}
\item Any column $\hat{W}_{:,j}$ of $\hat{W}$ belongs to the lattice $L=(\alpha_1\ZZ)\times\cdots\times(\alpha_n\ZZ)$ whose density is $\gamma(L)=\alpha^{-n}$
\item $U(W-\hat{W})\in \alpha|\Sigma_X|^{1/2n}\cdot \left[-\frac{1}{2},\frac{1}{2}\right)^{n\times a}$
\item If we further assume $U(W-\hat{W})\sim\Unif\left( \alpha|\Sigma_X|^{1/2n}\cdot \left[-\frac{1}{2},\frac{1}{2}\right)^{n\times a}\right)$ then
\begin{align}
&D_{\mathrm{WaterSIC}}=\frac{1}{an}\EE\|X^\top(W-\hat{W})\|^2=\frac{\alpha^2|\Sigma_X|^{1/n}}{12}.    \nonumber
\end{align}
\item $\hat{W}\in W+\alpha|\Sigma_X|^{1/2n}\cdot U^{-1}\left[-\frac{1}{2},\frac{1}{2}\right)^{n\times a}$
\item $Z_{\mathrm{SIC}}\in\m{A}^{-1}W+\tilde{U}\left[-\frac{1}{2},\frac{1}{2}\right)^{n\times a}$ where $\m{A}=\diag(\alpha_1,\cdots,\alpha_n)$ and $\tilde{U}=\alpha|\Sigma_X|^{1/2n}\cdot(U\m{A})^{-1}$ is an upper triangular matrix with unit determinant.
\end{enumerate}
\end{lemma}

The proof is immediate in light of Lemma~\ref{lem:FundCellSIC}. Using~\eqref{eq:RvsDensity}, we see that if the resulting output $Z_{\mathrm{SIC}}$ is encoded to bits via ball-shaping or entropy coding, the relation between $\alpha$ and quantization rate $R$ is
\begin{align}
\alpha\approx\sqrt{2\pi e \sigma_W^2  2^{-2R
}}.    
\end{align}
Substituting this into item 3 of Lemma~\ref{lem:SICAMGM}, gives~\eqref{eq:dwatersic_final}. See~\cite[Theorem 3.3]{lifar2026watersic} for a precise statement. Remarkably, the waterfilling choice of scaling factors~\eqref{eq:alphaAMGM} result in a lattice whose distortion under the very simple SIC algorithm is only a factor of $\frac{2\pi e}{12}\approx 1.4233$ from the lower bound~\eqref{eq:WMSElatLB}. The loss of this factor is due to the fact that, while the weighted error $U(W-\hat{W})$ lies within a convex set in isotropic position, this set is a cube. The lower bound~\eqref{eq:WMSElatLB} on the other hand is achieved with equality only if the convex set supporting the weighted error is a $\ell_2$-ball. The decision regions for the WaterSIC algorithm, as well as for the optimal lattice quantizer corresponding to $L=(\alpha_1\ZZ)\times(\alpha_2\ZZ)$ are illustrated in Figure~\ref{fig:tiling}.


\begin{figure}[t]
  \centering
  \includegraphics[width=\columnwidth]{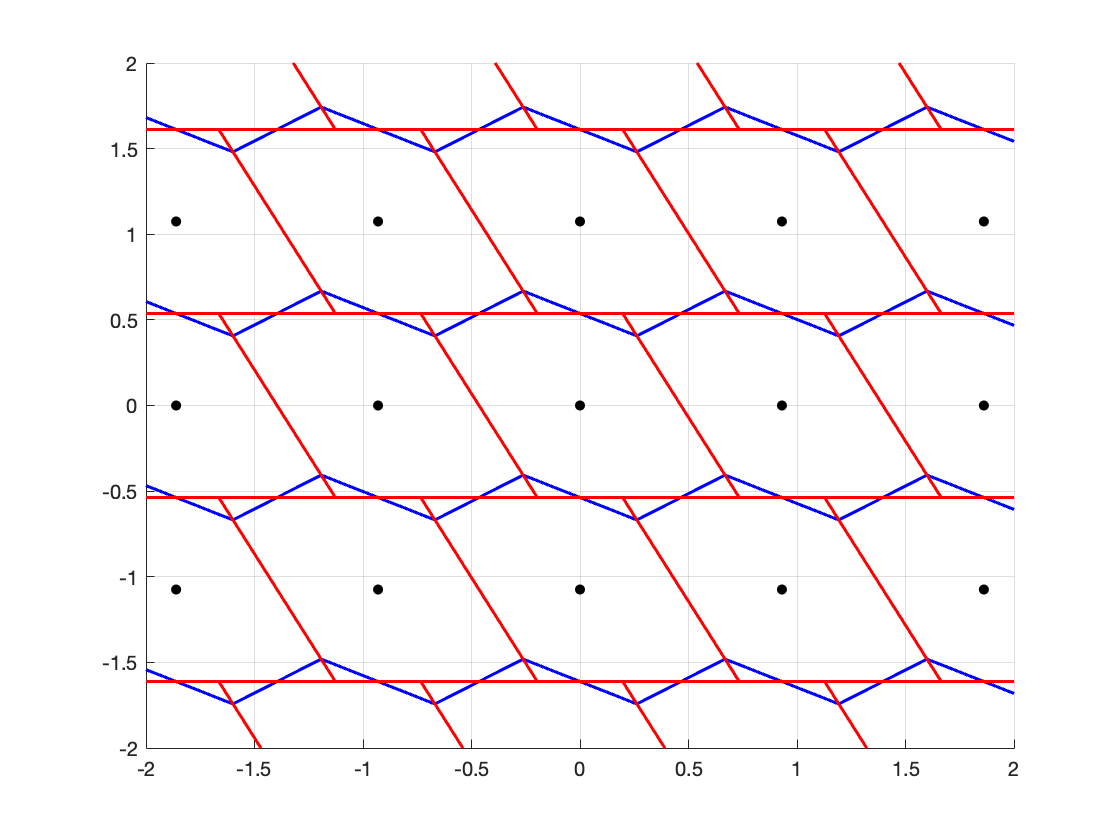}
  \caption{Illustration of the quantization regions for the optimal lattice quantizer (blue) and for the WaterSIC lattice quantizer (red) for $\Sigma_X=V \Lambda V^T$ where $V=[1~1;1~-1]/\sqrt{2}$ and $\Lambda=\diag(3,1)$. The lattice used for quantization is $L=(\alpha_1\ZZ)\times(\alpha_2 \ZZ)$ where $\alpha_1,\alpha_2$ are determined via~\eqref{eq:alphaAMGM}, with $\alpha=1$.}
  \label{fig:tiling}
\end{figure}

\begin{figure}[t]
  \centering
  \includegraphics[width=\columnwidth]{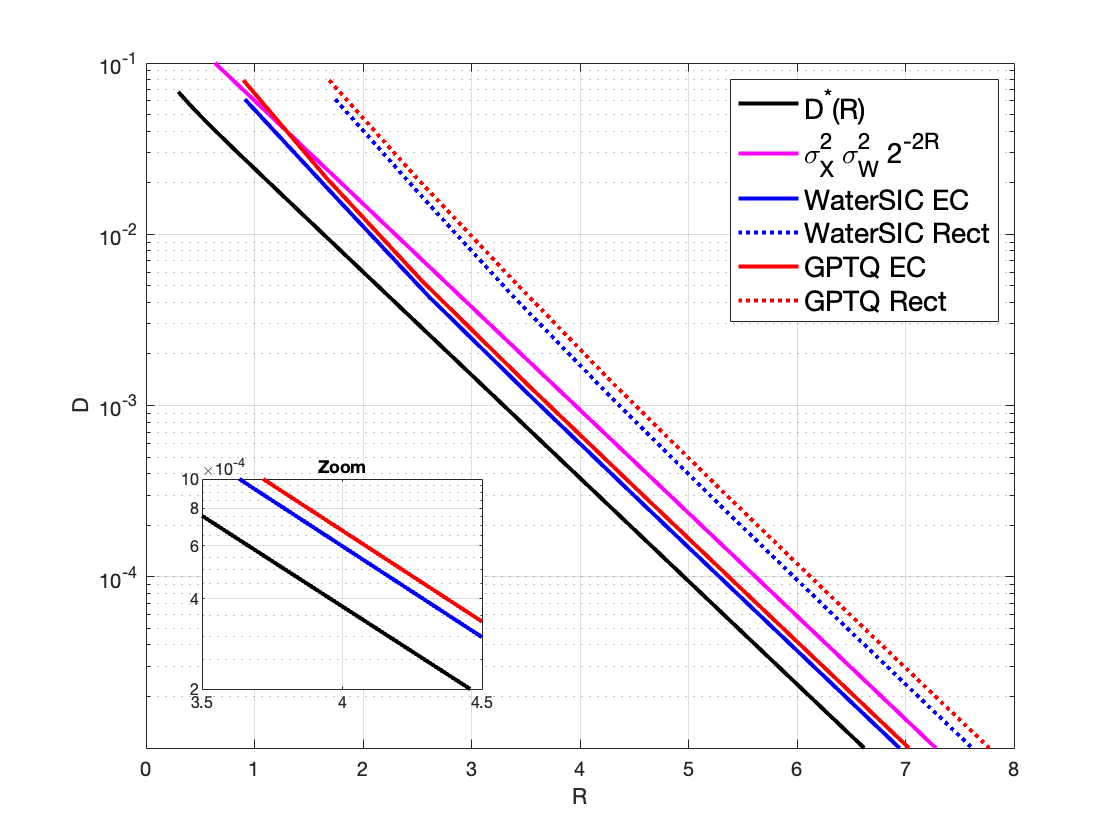}
  \caption{Performance of several weight-only quantization schemes against benchmark. Here
  $W\sim\m{N}(0,I_n)$ where $n=4096$ and $\Sigma_X$ is the empirical covariance matrix computed from  activation samples corresponding to $W_v$ in the $15$th layer of Llama3-8B.}
  \label{fig:sic_rd}
\end{figure}

\begin{figure*}[t] 
    \centering
    
    \begin{subfigure}[b]{0.48\textwidth}
        \centering
        \includegraphics[width=\linewidth]{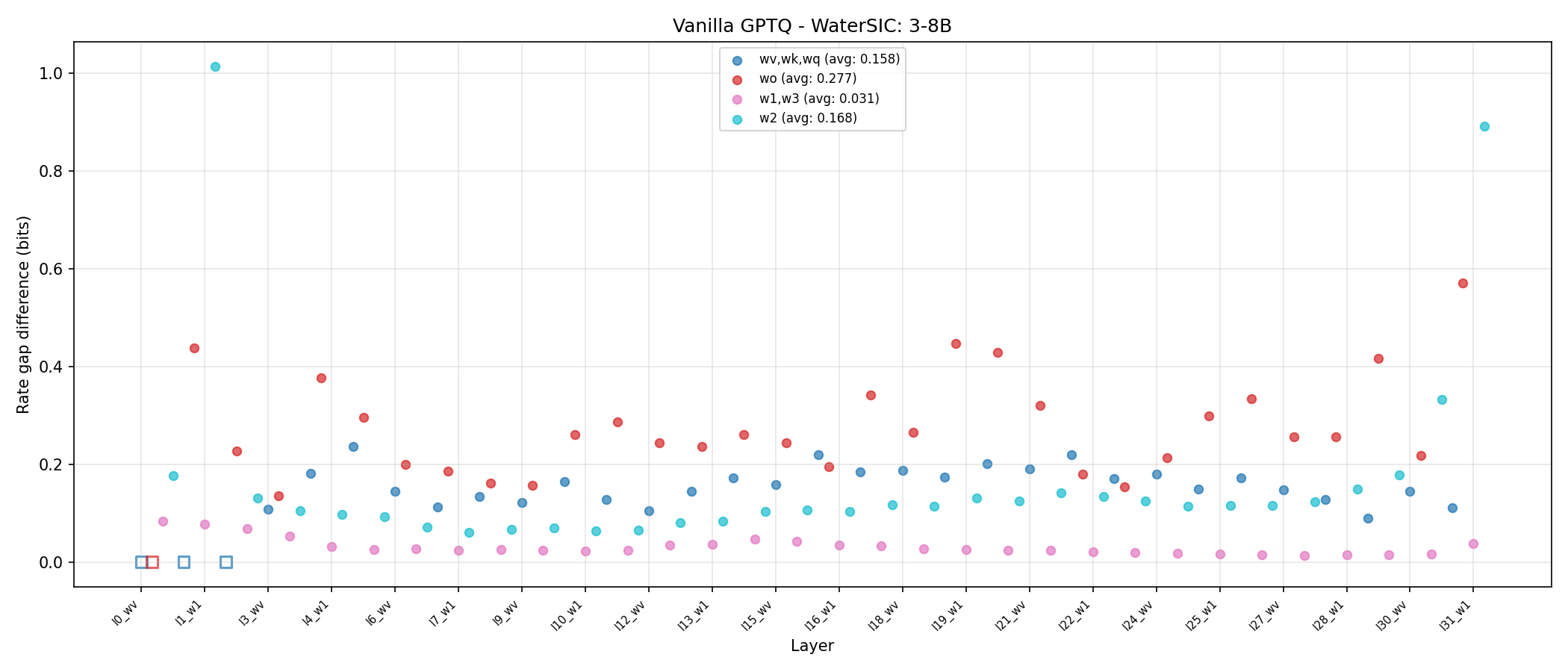}
        \caption{No random rotation.}
    \end{subfigure}
    \hfill 
    \begin{subfigure}[b]{0.48\textwidth}
        \centering
        \includegraphics[width=\linewidth]{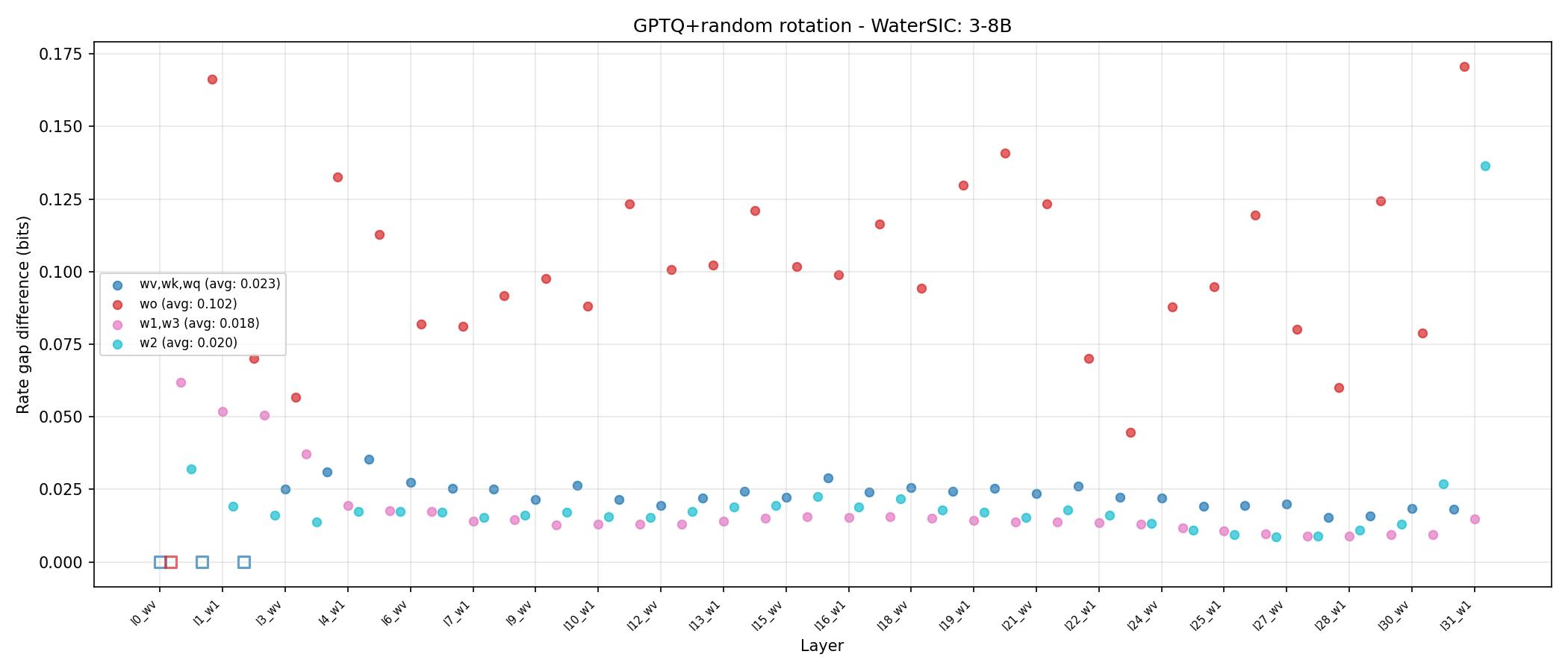}
        \caption{With random rotation.}
    \end{subfigure}
    
    \caption{Illustrating rate advantage of WaterSIC over SIC for $\Sigma_X$ of activations entering various layers of Llama-3-8B when
    processing Wikitext-2 dataset. Squares correspond to layers with singular $\Sigma_X$ which we
    skip. Note that WaterSIC achieves (on random gaussian $W$) identical performance with or
    without rotation.}
    \label{fig:llama_chol}
\end{figure*}

\textbf{Shaping:} There are several ways to represent $Z_{\mathrm{SIC}}$ in bits. One way that is simple, but quite inefficient, is using the shaping region
\begin{align}
\m{S}_{\text{rect}}&=[-q_1,q_1]\times\cdots\times[-q_n,q_n]\nonumber\\
&~~\text{where}~q_i=\|Z_{\text{SIC}i,:}\|_{\infty},~~\forall i\in[n].  \label{eq:rectShaping}
\end{align}
By definition, all entries of the $i$th row of $Z_{\mathrm{SIC}}$ are in $[-q_i,q_i]$, so overload never occurs. Furthermore, the encoder can describe the shaping region to the decoder by sending $\{q_i\}_{i=1}^n$, with negligible rate if the number of rows is large. And the quantization rate using $\tilde{S}_{\text{rect}}$ is $R_{\text{rect}}=\frac{1}{n}\sum_{i=1}^n\log(1+2q_i)$. Here, we have used the same reconstruction alphabet for an entire row of $W$. In general, one can decide on different ways to group elements that are described using the same reconstruction alphabet. While using the optimal group size can significantly boost performance, here we do not explore this. 

A better, but more complicated, way to perform shaping is via partitioning to cosets, much like in Forney's trellis shaping~\cite{forney1992trellis,eyuboglu1993lattice}. In particular, one can use a lattice $L_{\text{Shaping}}\subset\ZZ^n$, with $\covol(L_{\text{Shaping}})=2^{nR}$, such that $|\ZZ^n/L_{\text{Shaping}}|=2^{nR}$. In other words, $L_{\text{Shaping}}$ partitions $\ZZ^n$ to $2^{nR}$ cosets $\{z_i+L_{\text{Shaping}}\}$, where $z_1,\ldots,z_{2^{nR}}\in\ZZ^n$ are coset representatives.
In order to encode $Z_{\mathrm{SIC},:,j}$ (the $j$th column of $Z_{\mathrm{SIC}}$) the encoder describes the coset it belongs to, using $nR$ bits. The decoder outputs the coset member $z\in\ZZ^n$ for which $\sum_{i=1}^n \alpha_i^2 z_i^2$ is minimal. It is easy to verify that the reconstruction constellation corresponding to this scheme is $\ZZ^n\cap\tilde{\m{S}}_{\text{Lattice-Shaping}}$ where
\begin{align}
\tilde{\m{S}}_{\text{Lattice-Shaping}}=\bigg\{y\in\RR^n~&:~\sum_{i=1}^n\alpha_i^2 y_i^2 \leq \sum_{i=1}^n\alpha_i^2 (y_i-t_i)^2 \nonumber\\
&~~~\forall t\in L_{\text{Shaping}} \bigg\}.    
\end{align}
The difficulty here is that the decoder needs to find the coset member that minimizes the scaled energy, which may be computationally intensive unless $L_{\text{Shaping}}$ admits special properties. Some progress in finding $L_{\text{Shaping}}$ that admit efficient decoding on GPUs was made in~\cite{savkin2025nestquant,tseng2024qtip}. There are many other options one can consider for shaping, which we did not list here.

Finally, the encoder may simply use entropy coding (implemented via any lossless compression package. Note that modern GPU already include dedicated hardware for fast lossless compression/decompression) to describe the entries of $Z_{\mathrm{SIC}}$.

\subsection{Quantizing Llama-3-8B} 

In Figure~\ref{fig:sic_rd} we plot the rate-distortion curves
attained by WaterSIC and GPTQ without tailored spacing (that is, with $\alpha_i=\alpha$) under
entropy coding or rectangular shaping corresponding to shaping with $\m{S}_{\text{rect}}$
from~\eqref{eq:rectShaping}. We see that for high rates the ratio between WaterSIC EC and the
fundamental limit $D^*(R)$ from~\eqref{eq:waterfillingsolution} is indeed quite close to
$\frac{2\pi e}{12}\approx 1.4233$ (or 0.25 bit in rate). We also see that for this particular $\Sigma_X$, at high
resolution WaterSIC EC offers only limited gain with respect to GPTQ EC. However, the former is
provably close to $D^*(R)$ at high-resolution, whereas the latter is not, and therefore the gap
between them may be bigger for other layers, other LLMs, other calibration sets etc.

Let us investigate difference between GPTQ and WaterSIC (entropy coded versions of both) further.
Recall that in the high-rate regime, the rate advantage of WaterSIC over GPTQ is estimated as half logarithm
of the ratio between arithmetic and geometric mean of squared diagonal entries of $U$,
cf.~\eqref{eq:dsic_final} and~\eqref{eq:dwatersic_final}.  We illustrate this rate
advantage on Fig.~\ref{fig:llama_chol} for all layers of Llama-3-8B. We compare both the case of
using (Cholesky decomposition of the) original $\Sigma_X$ as well as $V^T \Sigma_X V$ for a random
orthogonal $V$. Recall that WaterSIC's performance is independent of rotation, and hence identical
in both figures, but the improvement over SIC is much smaller with random rotation. Let us
consider implications of this finding.

\textit{Cholesky factor interpretation.} Recall that $k$-th diagonal entry $U_{k,k}$ can be interpreted as follows. Think of
activations $(X_1,\ldots,X_n)$ as a stochastic process. Then, each sample $X_k$ can be
decomposed as 
$$ X_k = X_k^{\|} + X_k^{\perp}\,,$$
where $X_k^{\|}$ is the component that is a linear combination of $(X_1,\ldots,X_{k-1})$
(expressible in terms of $U_{k,j}, j = 1,\ldots,k-1$) and $X_k^{\perp}$ is the \emph{orthogonal
innovation} contained in $X_k$ compared to previous coordinates. $U_{k,k}^2 =
\EE[(X_k^{\perp})^2]$ is simply the energy of this innovation. Now, applying trace to the 
identity $\Sigma_X = U^\top U$ we get
$$ \sum_{k,j} U_{k,j}^2 = \sum_{j=1}^n \lambda_j\,. $$
Therefore, the average ${1\over n} \sum_k U_{k,k}^2$, that governs SIC's
performance~\eqref{eq:dsic_final}, equals ${1\over n} \sum_j \lambda_j$ minus the sum of squares of
off-diagonal entries of $U$. This implies, that GPTQ's performance is \textit{the worst} in the PCA
basis (where $U$ becomes diagonal), which is counter-intuitive. We also see that
the gap between $D_{iso}$ ($\Sigma_X$-oblivious encoder and decoder) and the waterfilling $D^*$,
which operates in the PCA-basis, is
much higher than for GPTQ vs WaterSIC (compare gaps on Fig.~\ref{fig:llama_ev} vs
Fig.~\ref{fig:llama_chol}).

\textit{Privileged basis.} On the other hand, as we see from
comparing two figures on on Fig.~\ref{fig:llama_chol} applying random rotation clearly improves
performance of GPTQ, thus decreasing $\sum_k U_{k,k}^2$. This implies that the original basis is
``closer'' to the diagonalizing PCA basis than a randomly rotated one.
This suggests that the standard basis for
$X$ is somehow priviliged among all other possible ones, an effect whose origins are most likely
due to operation of Adam~\cite{elhage2023mathematical}, which introduces axis aligned outliers in
activations.

\textit{Cholesky in random basis.} Can we estimate $\sum_k U_{k,k}^2$ from knowing the spectrum
$\{\lambda_j\}$ of $\Sigma_X$? The answer is affirmative if one considers applying random
rotation. One such estimate was offered (for a restricted class of $\Sigma_X$) by~\cite[Lemma
2]{chee2023quip}. However, one can get a more precise information about this sum.

Indeed, when we take $V$
to be a random orthogonal matrix and set $\Sigma_X = V^T \diag\{\lambda_j\} V$, quantities
$\{U_{k,k}\}$ become functions of a random matrix $V$. One can show that\footnote{Simply notice
that $\Sigma_X^{(k)} = U^{(k)\top} U^{(k)}$.} 
$$ U_{k,k}^2 = {|\Sigma_X^{(k)}|\over |\Sigma_X^{(k-1)}|}\,,$$
where $|\Sigma_X^{(k)}|$ is a determinant of a $k \times k$ principal minor of $\Sigma_X$. In dimensions of interest,
these random quantities concentrate quickly around their expectations, which can be computed
as\footnote{To see this, let $\tilde V$ be a $n\times k$ matrix consisting of first $k$ columns of
$V$, then from Cauchy-Binet $|\Sigma_X^{(k)}|=|(\Lambda^{1\over 2} \tilde V)^\top
(\Lambda^{1\over2} \tilde V)| =
\sum_{|S|=k} \prod_{i\in S} \lambda_i |\tilde V_{S\times [k]}|^2$. Taking expectation over $V$ we
have from symmetry $\EE[|\tilde V_{S\times [k]}|^2] = c$ for all subsets $S\subset [n]$. When all
$\lambda_i=1$ we must have $|\Sigma_X^{(k)}|=1$ and hence $c={n\choose k}^{-1}$.}
$$ \EE[|\Sigma_X^{(k)}|] = {1\over {n\choose k}} \sum_{|S|=k} \prod_{i\in S} \lambda_i\,.$$
Thus, we can approximate:
\begin{equation}\label{eq:ukk_approx}
	U_{k,k}^2 \approx {k \over n-k+1}  {\sum_{|S|=k} \prod_{i\in S} \lambda_i\over \sum_{|S|=k-1}
\prod_{i\in S} \lambda_i}\,.
\end{equation}
In particular, $U_{1,1}^2 \approx {1\over n} \sum_{j=1}^n \lambda_j$ and $U_{n,n}^2 \approx
{n\over \sum_{j=1}^n \lambda_j^{-1}}$ and the intermediate values smoothly decrease from the
arithmetic mean to harmonic mean. This fact is numerically illustrated on
Fig.~\ref{fig:random_chol}. Consequently, performance of GPTQ with random rotation can be
accurately estimated from combining~\eqref{eq:dsic_final} and~\eqref{eq:ukk_approx}. It is an
interesting open problem to estimate worst possible gap (over possible spectra $\lambda_j \ge
\epsilon$)
between GPTQ with rotation and WaterSIC (which in turn is $0.25$-bit away from
information-theoretically optimal waterfilling). 

\begin{figure*}[t] 
    \centering
        \includegraphics[width=.48\linewidth]{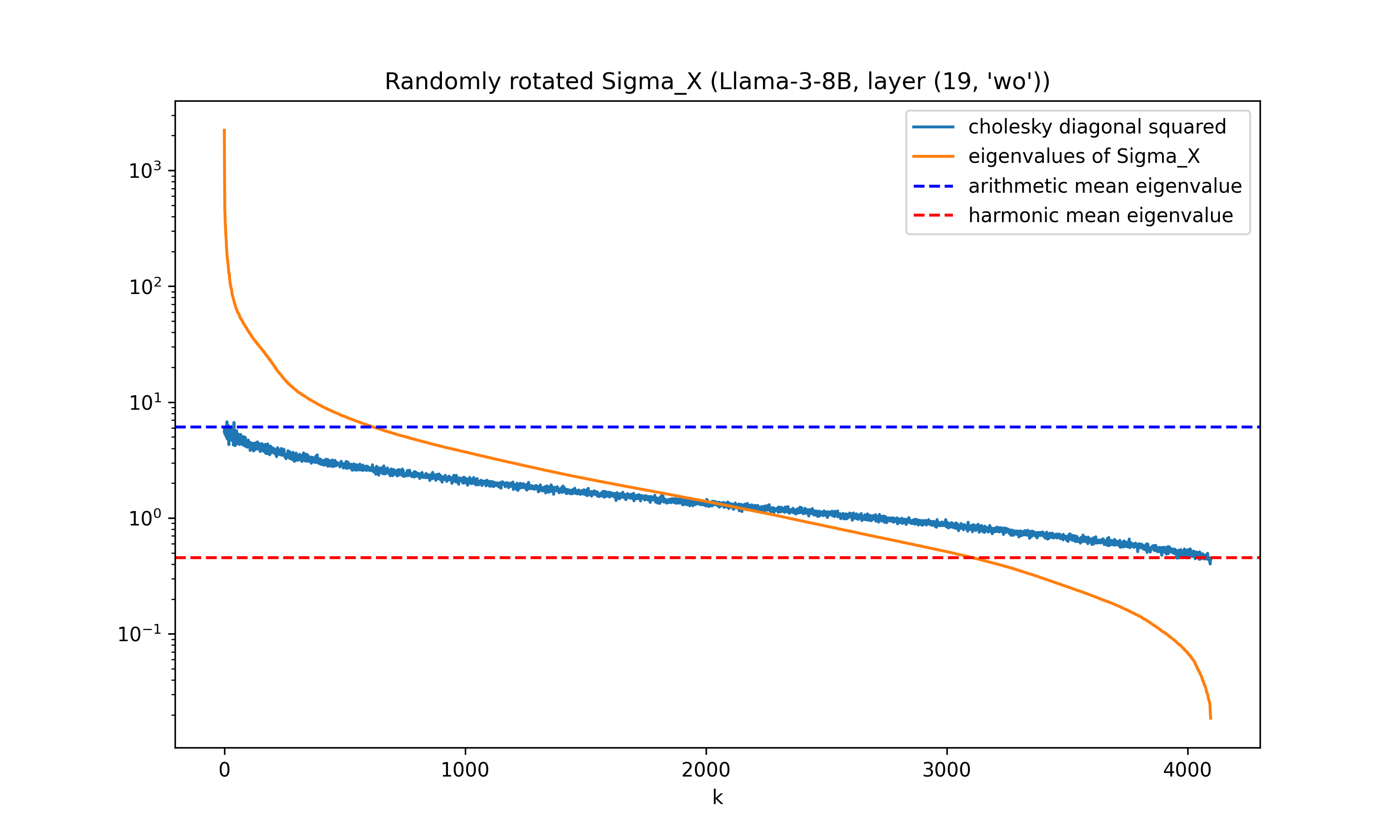}
    \caption{Illustrating Cholesky diagonals $U_{k,k}^2$ for a randomly rotated $V^\top \Sigma_X
    V$ and accuracy of approximation~\eqref{eq:ukk_approx} in terms of spectrum of $\Sigma_X$.}
    \label{fig:random_chol}
\end{figure*}

\textit{WaterSIC vs ``universal codebook''.}
We note that 
theoretical results from~\cite{HOP26}, discussed in Section~\ref{sec:universal_decoder}, demonstrated that
solving~\eqref{eq:WMSEquantproblem} optimally essentially attains full waterfilling solution even
in low rates, and furthermore does so with a fully $\Sigma_X$-oblivious codebook. Note that
WaterSIC chooses per-coordinate scales $\alpha_i \propto U_{i,i}^{-1}$ (Cholesky decomposition of) $\Sigma_X$ and, thus, requires
large $a \gg 1$ to amortize sending of $\alpha_i$'s in high precision (usually BF16). The results in ~\cite{HOP26} on the other hand, are valid for any $a\geq 1$.

\subsection{Future of weight-only quantization}

The high-resolution approximations~\eqref{eq:ShapedLatRD} and~\eqref{eq:ecdqRD} in
Subsection~\ref{subsec:latticequantHR}, combined with the analysis of WaterSIC in
Subsection~\ref{sec:watersic}, show that when WaterSIC followed by EC or  close-to-optimal
spherical shaping is used, we attain the optimal information theoretic distortion
from~\eqref{eq:Dhr} up to a multiplicative gap of $\frac{2\pi e}{12}$. Figure~\ref{fig:sic_rd}
numerically confirms the accuracy of the high-resolution approximations. Furthermore,
Figure~\ref{fig:llama_chol} shows that with random rotation even the vanilla GPTQ algorithm, which
is completely equivalent to the widely used GPTQ/LDLQ algorithms, achieves distortion only
slightly greater than WaterSIC, and is therefore nearly optimal. In light of this, one may wonder
whether the practical schemes are already so good that there is no further room for improvement.
We claim that this is not the case, and list several important directions for improved
quantization schemes:

First, in this survey paper we restricted attention to the high-rate regime. At small
    quantization rates (say $0.5-2$ bits per entry) many other effects will become important:
    necessity of Johnson-Lindenstrauss dimensionality reduction, as shown
    in~\cite{ordentlich2024optimal}, shrinkage and other low-rate effects. Proper shaping and
    exploiting structure of $W$ matrices will be paramount. Indeed, at low
    rates~\cite{lifar2026watersic} shows WaterSIC to have massive advantage over GPTQ, even if the
    latter is Huffman coded.

Second, even our analysis for the high-resolution regime relied on applying EC or near-optimal
    spherical shaping after quantization. While EC over $\ZZ$ can be efficiently implemented, even
    on modern GPUs EC decompression is about an order of magnitude slower than loading the
    uncompressed entries. Thus, the usefulness of EC for running quantized models on a GPU is
    somewhat questionable. On the other hand, simple shaping methods that can be easily made
    efficient, e.g. rectangular shaping, may be highly suboptimal (see Figure~\ref{fig:sic_rd}).
    The problem of shaping schemes that are amenable to fast GPU implementation is at its infancy~\cite{NestQuant}.

Third, while the $\frac{2\pi e}{12}$ high-resolution gap-to-optimality of WaterSIC+EC is relatively small, reducing it further is of interest. This gap stems from the sub-optimality of the integer lattice $\ZZ^n$ as a quantizer. Generally, we can decrease this gap by combining WaterSIC with jointly quantizing $d$ entries from the same row of $W$ to a lattice $L'\subset\RR^d$, followed by EC. If we use the optimal lattice quantizer in $\RR^{d'}$ we can reduce the $\frac{2\pi e}{12}$ gap to optimality to $2\pi e \cdot G_d$ where $G_d$ is the optimal normalized second moment in dimension $d$. See~\cite[Table I]{AgrellAllen22} for a list of the best known NSM for $1\leq d\leq 48$ as of 2023 (some improvements were reported in certain dimensions since then). The downside of replacing $\ZZ$ with a higher-dimensional lattice is that EC of the quantizer's output becomes more challenging as its cardinality is much larger.

Fourth, implementation of WaterSIC relies on the fact that sending per-channel scales $\alpha_i$
can be amortized freely, which is only the case when one has many output neurons in a linear
layer (large $a \gg 1$).

\section*{Acknowledgment}
We thank Alina Harbouzova, Egor Lifar, and Semyon Savkin (MIT EECS) for numerous discussions and help with obtaining calibration matrices for Llama-3-8B.

\bibliographystyle{IEEEtran}
\bibliography{IPRD_Bib}

\end{document}